\documentclass[twoside,11pt]{article}
\usepackage{indentfirst}
\usepackage{blindtext}
\usepackage{amsmath,amssymb,amsfonts}
\usepackage{enumitem}
\usepackage{algorithm}
\usepackage{lineno}
\modulolinenumbers[5]
\usepackage{multirow}
\usepackage[font=small,labelfont=bf]{caption}
\usepackage{subcaption}
\usepackage{tikz}
\usepackage{algpseudocode}
\usepackage{amsthm}
\usepackage{graphicx}
\usepackage{textcomp}
\usepackage{url}
\usepackage{color}
\usepackage{booktabs}  
\usepackage{lipsum}
\usepackage{epstopdf}
\usepackage{graphicx}
\usepackage{epstopdf}
\usepackage{float}
\usepackage{tabularx}

\usepackage{amsopn}
\DeclareMathOperator{\diag}{diag}

\newcommand{\tol}{\mathtt{tol}}

\usepackage{cancel}
\usepackage{mathrsfs}
\usepackage{mathdots}
\usepackage{euscript}
\usepackage{amscd}
\usepackage{placeins}
\usepackage{natbib}
 \usepackage[abbrvbib, preprint]{jmlr2e}
\usepackage{titlesec}  

\titleformat{\section}
{\normalfont\Large\bfseries}{\thesection.}{0.5em}{}

\newcommand{\LzeroneSVM}{$L_{0/1}$-SVM}
\newtheorem{assumption}{Assumption}

\newcommand{\bmat}{\left[ \begin{matrix}}
	\newcommand{\emat}{\end{matrix} \right]}
\newcommand{\innerprod}[2]{\langle{#1},\,{#2}\rangle}

\DeclareMathOperator{\argmin}{argmin}
\DeclareMathOperator{\prox}{prox}

\newcommand{\Rbb}{\mathbb R}
\newcommand{\Hbb}{\mathbb H}

\newcommand{\Fbb}{\mathbb F}
\newcommand{\Nbb}{\mathbb N}
\newcommand{\xb}{\mathbf  x}
\newcommand{\yb}{\mathbf  y}
\newcommand{\sbf}{\mathbf  s}  
\newcommand{\zb}{\mathbf  z}
\newcommand{\wb}{\mathbf  w}
\newcommand{\vb}{\mathbf  v}

\newcommand{\fb}{\mathbf  f}

\newcommand{\db}{\mathbf  d}

\newcommand{\ub}{\mathbf  u}
\newcommand{\rb}{\mathbf  r}

\newcommand{\oneb}{\mathbf 1}
\newcommand{\zerob}{\mathbf 0}

\newcommand{\rhob}{\boldsymbol{\rho}}
\newcommand{\thetab}{\boldsymbol{\theta}}

\newcommand{\lambdab}{\boldsymbol{\lambda}}

\newcommand{\test}{\text{test}}
\DeclareMathOperator{\sign}{sign}
\newcommand{\stepfunc}{\|(\cdot)_+\|_0}

\DeclareMathOperator{\rank}{rank}

\newcommand{\Acal}{\mathcal{A}}

\newcommand{\Lcal}{\mathcal{L}}

\newcommand{\Ical}{\mathcal{I}}

\newcommand{\Xcal}{\mathcal{X}}
\newcommand{\Kcal}{\mathcal{K}}

\newcommand{\Ocal}{\mathcal{O}}

\usepackage{lastpage}

\ShortHeadings{MKL-$L_{0/1}$-SVM}{Zhu and Shi}
\firstpageno{1}

\begin{document}

\title{MKL-$L_{0/1}$-SVM\thanks{A preliminary version of this paper \citep{shi2023admm} will be presented at the 62nd IEEE Conference on Decision and Control (CDC 2023).}}

\author{\name Bin Zhu and Yijie Shi \\
	\email zhub26@mail.sysu.edu.cn and shiyj27@mail2.sysu.edu.cn\\
	\addr School of Intelligent Systems Engineering\\
	Sun Yat-sen University\\
	Gongchang Road 66, 518107 Shenzhen, China
       }

\editor{My editor}

\maketitle

\begin{abstract}


This paper presents a Multiple Kernel Learning (abbreviated as MKL) framework for the Support Vector Machine (SVM) with the $(0, 1)$ loss function. Some KKT-like first-order optimality conditions are provided and then exploited to develop a fast ADMM algorithm to solve the nonsmooth nonconvex optimization problem. Numerical experiments on real data sets show that the performance of our MKL-$L_{0/1}$-SVM is comparable with the one of the leading approaches called SimpleMKL developed by Rakotomamonjy, Bach, Canu, and Grandvalet [Journal of Machine Learning Research, vol.~9, pp.~2491--2521, 2008].
\end{abstract}


\section{Introduction}


The support vector machine (SVM) is an important tool in machine learning with numerous applications \citep{vapnik2000nature, theodoridis2020machine}.
The theory can be traced back to the seminal work of \cite{cortes1995support}. 
In the basic setting, the SVM deals with the binary classification task where
a data set $\{(\xb_i,y_i) \in \Rbb^n\times \{-1, 1\}: i\in\Nbb_m\}$ is given and one is asked to construct a function $\tilde{f}(\xb)$ in order to separate the two classes of feature vectors $\xb_i$ with labels $y_i=-1$ or $1$, and to predict the labels of unseen feature vectors.
To this end, the SVM first lifts the problem to a \emph{reproducing kernel Hilbert space}\,\footnote{The theory of RKHS goes back to \cite{aronszajn1950theory} and many more, see e.g., \cite{paulsen2016introduction}.} (RKHS) $\Hbb$, in general infinite-dimensional and equipped with a \emph{positive definite} kernel function $\kappa : \Rbb^n \times \Rbb^n \to \Rbb$, via the feature mapping
\begin{equation}
\xb \mapsto \phi(\xb) := \kappa(\cdot,\xb) \in \Hbb,
\end{equation}
and then considers decision (or discriminant) functions of the form
\begin{equation}\label{discrimi_func}
\tilde{f}(\xb) = b + \innerprod{w}{\phi(\xb)}_{\Hbb} = b+w(\xb),
\end{equation}
where $b\in\Rbb$, $w\in\Hbb$, $\innerprod{\cdot}{\cdot}_{\Hbb}$ the inner product associated to the RKHS $\Hbb$, and the second equality is due to the so-called \emph{reproducing property}. 
It is not difficult to see that such a decision function is in general nonlinear in $\xb$, but is indeed linear with respect to $\phi(\xb)$ in the feature space $\Hbb$.
Once $\tilde{f}$ is determined, the label of $\xb$ is assigned via $y(\xb)=\sign [\tilde{f}(\xb)]$ where $\sign(\cdot)$ is the sign function which gives $+1$ for a positive number, $-1$ for a negative number, and left undefined at zero.

It remains to estimate the unknown quantities $b$ and $w$ in \eqref{discrimi_func}, and this is done via the solution of
the unconstrained optimization problem
\begin{equation}\label{optim_infinit_dim}
\min_{\substack{w\in\Hbb, \, b\in\Rbb, \\ \tilde f(\cdot) = w(\cdot) + b}}\quad \frac{1}{2}\|w\|_{\Hbb}^2 + C \sum_{i} \Lcal(y_i, \tilde f(\xb_i)),
\end{equation}
where $\|w\|_{\Hbb}^2 = \innerprod{w}{w}_{\Hbb}$ is the squared norm of $w$ induced by the inner product, $\Lcal(\cdot,\cdot)$ is a suitable loss function, and $C>0$ is a regularization parameter.
For the choice of $\Lcal$, \cite{cortes1995support} suggested the $(0,1)$ loss function, also called $L_{0/1}$ loss in \cite{wang2021support}, for quantifying the error of classification which essentially counts the number of misclassified samples. More precisely, we take
\begin{equation}\label{L_0/1-loss}
\Lcal_{0/1} (y, \tilde{f}(\xb)) := H(1-y\tilde{f}(\xb))
\end{equation}
where $H$ is the Heaviside unit step function
\begin{equation}\label{func_unit_step}
H(t) = \begin{cases}
1, & t>0\\
0, & t\leq 0,
\end{cases}
\end{equation}
see the left panel Fig.~\ref{fig:step_and_ReLU}.
\begin{figure}
	\begin{subfigure}[H]{0.5\textwidth}
		\centering
		\begin{tikzpicture}
		\draw[line width=0.5pt][->](-2,0)--(2,0)node[left,below,font=\tiny]{$t$};
		\draw[line width=0.5pt][->](0,-0.8)--(0,1.4)node[right,font=\tiny]at(0.3,1.2){$H(t)=|t_+|_0$};
		\node[below,font=\tiny] at (0.15,0){0};
		\node[right,below,font=\tiny]at(0.14,1){1};
		\draw[color=blue, thick,smooth,domain=0:1.8]plot(\x,1);
		\draw[color=blue, thick,smooth,domain=-1.8:0.0]plot(\x,0);
		\draw[color=blue,fill=blue,smooth]circle(0.03);
		\draw[fill = white](0,1)circle(0.03);
		\end{tikzpicture}
	\end{subfigure}
	\hfill
	\begin{subfigure}[H]{0.5\textwidth}
		\centering
		\begin{tikzpicture}
		\draw[line width=0.5pt][->](-2,0)--(2,0)node[left,below,font=\tiny]{$t$};
		\draw[line width=0.5pt][->](0,-0.8)--(0,1.4)node[right,font=\tiny]at(0.6,1.2){$t_+$};
		\node[below,font=\tiny] at (0.15,0){0};
		\draw[color=orange, thick,smooth,domain=0:1.2]plot(\x,\x);
		\draw[color=orange, thick,smooth,domain=-1.8:0.0]plot(\x,0);
		\draw[color=orange,fill=blue,smooth]circle(0.03);
		\end{tikzpicture}
	\end{subfigure}
	\caption{\emph{Left}: the unit step function. \emph{Right}: the function $t_+ := \max\{0, t\}$.}
	\label{fig:step_and_ReLU}
\end{figure}
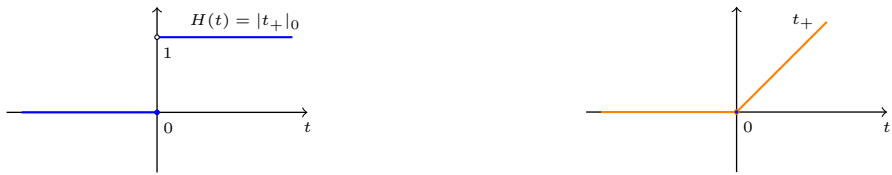
However, it was also pointed out in \cite{cortes1995support} that the resulting optimization problem is \emph{NP-complete, nonsmooth, and nonconvex} due to the discrete nature of the $(0,1)$ loss.
The subsequent research focused on designing other (easier) loss functions, notably convex ones like the \emph{hinge} loss in which the step function $H(t)$ in \eqref{L_0/1-loss} is replaced with $t_+$ in the right panel of Fig.~\ref{fig:step_and_ReLU}.
However, the hinge loss has the drawback of being unbounded and thus more sensitive to outliers in the data. Other convex and nonconvex loss functions for the SVM can be found in the excellent literature review in \citet[Section~2]{wang2021support}.



In recent years, there is a resurging interest in the original SVM problem with the $(0,1)$ loss, abbreviated as ``$L_{0/1}$-SVM'', following theoretical and algorithmic developments for optimization problems with the ``\emph{$\ell_0$ norm}'', see e.g., \cite{nikolova2013description,zhou2021quadratic,zhou2021global} and the references therein. In particular, \cite{wang2021support} proposed KKT-like optimality conditions for the \emph{linear} $L_{0/1}$-SVM optimization problem and an efficient Alternating Direction Method of Multipliers (ADMM) solver to obtain an \emph{approximate} solution whose performance is competitive among other SVM models.
In this work, we draw inspiration from the aforementioned papers and present a \emph{multiple-kernel} version of the theory in which the ambient functional space has a richer structure than the usual Euclidean space. More precisely, we shall formulate the $L_{0/1}$-SVM problem in the MKL context, see e.g., \citet{bach2004multiple,lanckriet2004learning,sonnenburg2006large,rakotomamonjy2008simplemkl}. Clearly, the MKL framework can offer much more flexibility than the single-kernel formulation by letting the optimization algorithm determine the best combination of different kernel functions. In this sense, our results represent a substantial generalization of the work in \cite{wang2021support} while maintaining the core features of the linear \LzeroneSVM.

The contributions of this paper are described next.
We show that the optimal solutions of our MKL-\LzeroneSVM\ problem can also be characterized by a set of KKT-like conditions. These conditions can be readily exploited in designing an ADMM algorithm  to solve the optimization problem despite the difficulties induced by the discontinuity and nonconvexity of the $L_{0/1}$ loss. Two working sets, one for the data and the other for the kernel combination, are employed to enhance the solution speed and to render sparsity in the solution. Moreover, any limit point generated by the algorithm is proved to be locally optimal. At last, numerical simulations show that our MKL-\LzeroneSVM\ is also competitive in comparison with the SimpleMKL approach in \cite{rakotomamonjy2008simplemkl}.

The remainder of this paper is organized as follows. Section~\ref{sec:prob} reviews the classic \LzeroneSVM\ in the single-kernel case 
and discusses the MKL framework.
Section~\ref{sec:optimality} establishes the optimality theory for the MKL-\LzeroneSVM\ problem, 
and in Section~\ref{sec:algorithm} we propose an ADMM algorithm to solve the optimization problem. Finally, numerical experiments and concluding remarks are provided in Sections~\ref{sec:sims} and \ref{sec:Conclusion},  respectively.

\subsection*{Notation}

$\Rbb_+ := \{x\in\Rbb : x\geq 0\}$ denotes the set of nonnegative reals, and $\Rbb_+^n :=\Rbb_+\times\cdots\times \Rbb_+$ the $n$-fold Cartesian product. $\Nbb_m:=\{1, 2, \dots, m\}$ is a finite index set for the data points and $\Nbb_L:=\{1, 2, \dots, L\}$ for the kernels.
Throughout the paper, the summation variables $i\in\Nbb_m$ is reserved for the data index, and $\ell\in\Nbb_L$ for the kernel index. We write $\sum_{i}$ and $\sum_{\ell}$ in place of $\sum_{i=1}^m$ and $\sum_{\ell=1}^L$ to simplify the notation.

\section{Problem Formulation}\label{sec:prob}


In all kernel-based methods, an important issue is to select a suitable kernel and its parameter which is also known as \emph{hyperparameter}. Such a selection can be done via cross-validation once the functional form of the kernel is specified. In this context, multiple kernel learning provides an alternative in which one employs a set of different kernels $\{\kappa_\ell : \ell=1,\dots, L\}$ and considers the SVM problem in the RKHS with a linearly combined kernel $\sum_{\ell}d_\ell\kappa_\ell$.  Each $d_\ell\geq 0$ is a free parameter included into the optimization problem. In other words, one seeks to simultaneously find the best combination of the kernels and the optimal decision function via the solution of the MKL problem.

Before formally stating our MKL optimization problem for the \LzeroneSVM, we shall first briefly continue the discussion on the single-kernel problem in the Introduction, and then describe the necessary functional space setup borrowed from \cite{rakotomamonjy2008simplemkl}.


\subsection{More on the Single-Kernel Case}

In this subsection, we further discuss the single-kernel version of the $L_{0/1}$-SVM along the lines of \cite{shi2023admm} and set up the notation.
The optimization problem \eqref{optim_infinit_dim} is cast on the RKHS $\Hbb$ which could be infinite-dimensional. However, one can appeal to the celebrated \emph{representer theorem} \citep{kimeldorf1971some} to reduce the problem to \emph{finite} dimensions. More precisely, by the \emph{semiparametric} representer theorem \citep{scholkopf-smola2001earning}, any minimizer of \eqref{optim_infinit_dim} must have the form
\begin{equation}\label{discrimi_func_parametric}
	\tilde f(\cdot) = \sum_{i} w_i \kappa(\,\cdot\,, \xb_i) + b,
\end{equation}
so that the desired function $w(\cdot)$ is completely parametrized by the $m$-dimensional vector $\wb=[w_1,\dots,w_m]$ which defines the the linear combination of the \emph{kernel sections} $\kappa(\,\cdot\,, \xb_i)$. After some algebra involving the \emph{kernel trick}, we obtain a finite-dimensional optimization problem
\begin{equation}\label{optim_finit_dim}
	\min_{\substack{\wb\in\Rbb^m, \, b\in\Rbb}}\ \ J(\wb,b) := \frac{1}{2} \wb^\top K \wb + C \|(\oneb - A\wb -b\yb)_+\|_0
\end{equation}
where,
\begin{itemize}
	\item $K=K^\top$ is the \emph{kernel matrix}
	\begin{equation}
		\bmat \kappa(\xb_1, \xb_1) & \cdots & \kappa(\xb_1,\xb_m)\\
		\vdots & \ddots & \vdots \\
		\kappa(\xb_m,\xb_1) & \cdots & \kappa(\xb_m,\xb_m) \emat \in \Rbb^{m\times m}
	\end{equation}
	which is \emph{positive semidefinite} by construction, 
	\item $\oneb\in\Rbb^m$ is a vector whose components are all $1$'s, 
	\item $\yb = [y_1,\dots, y_m]^\top$ is the vector of labels, 
	\item the matrix $A = D_{\yb} K$ is such that $D_{\yb} = \diag(\yb)$ is the diagonal matrix whose $(i,i)$ entry is $y_i$, 
	\item the function $t_+ = \max\{0, t\}$ takes the positive part of the argument when applied to a scalar
	(right panel of Fig.~\ref{fig:step_and_ReLU}), and $\vb_+ := [(v_1)_+, \dots, (v_m)_+]^\top$ represents componentwise application of the scalar function,
	\item $\|\vb\|_0$ is the $\ell_0$ norm\,\footnote{Since ``$\ell^p$-norms'' are not \emph{bona fide} norms for $0\leq p<1$, it may be better called $\ell_0$ pesudonorm.} that counts the number of nonzero components in the vector $\vb$.
\end{itemize}
Clearly, the composite function $\|\vb_+\|_0$ counts the number of positive components in $\vb$. For a scalar $t$, it coincides with the step function in \eqref{func_unit_step}.

\begin{remark}\label{rem_redu}
	
	The linear \LzeroneSVM\ studied in \cite{wang2021support} can be viewed as a special case of the above problem where one employs a \emph{homogeneous polynomial} kernel
		$\kappa(\xb, \yb) = (\xb^\top \yb)^d$
	with the degree parameter $d=1$. 
		In general, if the kernel function $\kappa(\xb,\yb)$ induces a \emph{finite-dimensional} RKHS $($as it is the case for polynomial kernels$)$, then the kernel matrix $K$, which is in fact a \emph{Gram matrix}, must be \emph{rank-deficient} when $m$ is sufficiently large. In such a situation, we have the rank factorization $K = X^\top X$ such that $X\in\Rbb^{r\times m}$ with $r=\rank K$. For example, in the previous case of $\kappa(\xb, \yb) = \xb^\top \yb$, we have $X=[\xb_1, \dots, \xb_m]\in\Rbb^{n\times m}$ if such an $X$ has full row rank. Then after a change of variables $\tilde{\wb}=X\wb$, the optimization problem \eqref{optim_finit_dim} further reduces to
		\begin{equation}\label{optim_finit_dim_redu}
			\min_{\substack{\tilde\wb\in\Rbb^r, \, b\in\Rbb}}\ \ \tilde J(\tilde\wb,b) := \frac{1}{2} \|\tilde\wb\|^2 + C \|(\oneb - \tilde A\tilde\wb -b\yb)_+\|_0
			\end{equation}
		where $\tilde{A}=D_{\yb} X^\top\in\Rbb^{m\times r}$ is a tall and thin matrix. This is almost exactly the problem studied in \cite{wang2021support}.
\end{remark}


For reasons discussed in Remark~\ref{rem_redu}, in the remaining part of this paper, we shall always assume that the kernel matrix $K$ is \emph{positive definite}. 
This assumption holds in particular, for the \emph{Gaussian} kernel 
\begin{equation}\label{Gaussian_kernel}
	\kappa(\xb,\yb) = \exp \left( -\frac{\|\xb-\yb\|^2}{2\sigma^2}\right),
\end{equation}
where $\sigma>0$ is a hyperparameter, see \cite{slavakis2014online}.
In such a case, the matrix $A=D_{\yb} K$ in \eqref{optim_finit_dim} is also \emph{invertible} since $D_{\yb}$ is a diagonal matrix
whose diagonal entries are the labels $-1$ or $1$. Indeed, we have $D_{\yb}^2 = D_{\yb}^\top D_{\yb}=I$.


\subsection{Functional Space for MKL}

For each $\ell\in\Nbb_L$, let $\Hbb_\ell$ be a RKHS of functions on $\Xcal\subset \Rbb^n$ with the kernel $\kappa_\ell(\cdot,\cdot)$ and the inner product $\innerprod{\cdot}{\cdot}_{\Hbb_\ell}$. Moreover, take $d_\ell\in\Rbb_+$, and define a Hilbert space $\Hbb_\ell'\subset \Hbb_\ell$ as
\begin{equation}
	\mathbb{H}_\ell^\prime :=\left\{ f\in \mathbb{H}_\ell : \frac{\|f\|_{\mathbb{H}_\ell}}{d_\ell}<\infty\right\}
\end{equation}
endowed with the inner product
\begin{equation}
	\left< f, g \right>_{\mathbb{H}_\ell^\prime}=\frac{\left< f, g \right>_{\Hbb_\ell}}{d_\ell}.
\end{equation}
We use the convention that ${x}/{0}=0$ if $x=0$ and $\infty$ otherwise. This means that, if $d_\ell=0$ then a function $f\in \mathbb{H}_\ell$ belongs to the subspace $\mathbb{H}_\ell^\prime$ only if $f=0$. In such a case, $\mathbb{H}_\ell^\prime$ becomes a trivial space containing only the null element.
Within this framework,  $\Hbb_\ell^\prime$ is a RKHS with the kernel $\kappa_\ell'(\xb,\yb)=d_\ell \kappa_\ell(\xb,\yb)$ since
\begin{equation}
		\forall f\in \Hbb_\ell^\prime\subset \Hbb_\ell,\ f(\textbf{x}) =\left<f(\cdot),\kappa_\ell(\xb,\cdot)\right>_{\Hbb_\ell} 
		 =\frac{1}{d_\ell} \left<f(\cdot),d_\ell\kappa_\ell(\xb,\cdot)\right>_{\Hbb_\ell} 
		=\left<f(\cdot),d_\ell\kappa_\ell(\xb,\cdot)\right>_{\Hbb_\ell^\prime}.
\end{equation}

Define $\Fbb:=\Hbb_1'\times \Hbb_2' \times \cdots \times \Hbb_L'$ as the Cartesian product of the RKHSs $\{\Hbb_\ell'\}$, which is itself a Hilbert space with the inner product
\begin{equation}
	\innerprod{(f_1, \dots, f_L)}{(g_1, \dots, g_L)}_\Fbb = \sum_{\ell} \innerprod{f_\ell}{g_\ell}_{\Hbb_\ell'}.
\end{equation}
Let $\Hbb:=\bigoplus_{\ell=1}^L \Hbb_\ell'$ be the \emph{direct sum} of the RKHSs $\{\Hbb_\ell'\}$, which is also a RKHS with the kernel function
\begin{equation}\label{combina_kernels}
	\kappa(\xb,\yb) = \sum_{\ell} d_\ell \kappa_\ell(\xb,\yb),
\end{equation}
see \cite{aronszajn1950theory}.
Moreover, the squared norm of $f\in \Hbb$ is known as
\begin{equation}\label{sqr_norm_direct_sum}
	\begin{aligned}
		\|f\|^2_{\Hbb} = \min \left\{ \sum_{\ell} \|f_\ell\|^2_{\Hbb_\ell'} = \sum_{\ell} \frac{1}{d_\ell} \|f_\ell\|^2_{\Hbb_\ell} : f=\sum_{\ell} f_\ell\
		 \text{such that}\ f_\ell\in\Hbb_\ell' \right\}.
	\end{aligned}
\end{equation}
The vector $\db=[d_1,\dots,d_L]^\top\in\Rbb_+^L$ is seen as a tunable parameter for the linear combination of kernels $\{\kappa_\ell\}$ in \eqref{combina_kernels}.

\subsection{The MKL-$L_{0/1}$-SVM Problem}

We take inspiration from \cite{rakotomamonjy2008simplemkl} and formulate the MKL-$L_{0/1}$-SVM problem as
\begin{subequations}\label{opt_MKL_01}
	\begin{align}
		& \underset{\substack{\fb=(f_1, \dots, f_L)\in \Fbb \\ \db\in\Rbb^L,\ b\in\Rbb}}{\min}
		& & \frac{1}{2} \sum_{\ell} \frac{1}{d_\ell} \|f_\ell\|^2_{\Hbb_\ell} + C\sum_{i} \Lcal_{0/1}(y_i, \tilde f(\xb_i)) \nonumber \\
		& \qquad\text{s.t.}
		& & d_\ell \geq 0,\ \ell\in\Nbb_L \label{d_l_simplex_01} \\ 
		& & & \sum_{\ell} d_\ell =1 \label{d_l_simplex_02} \\
		& & & \tilde f(\cdot) = \sum_{\ell} f_\ell(\cdot) + b \nonumber
	\end{align}
\end{subequations}
where $C>0$ is a regularization parameter. The first (regularization) term in the objective function is chosen so due to its convexity, see \citet[Appendix~A.1]{rakotomamonjy2008simplemkl}, which facilitates theoretical analysis. Moreover, constraints \eqref{d_l_simplex_01} and \eqref{d_l_simplex_02} define the standard $(L-1)$-simplex which ensures that the combined kernel is again positive definite.
Also, the compactness property of the simplex is useful in the proof that a minimizer exists.

The last constraint in \eqref{opt_MKL_01} can be safely eliminated by a substitution into the objective function. 
Next, define a new variable $\ub\in\Rbb^m$ by letting $u_i := 1-y_i(f(\xb_i)+b)$ where $f=\sum_{\ell} f_\ell$. Then the next lemma is a fairly straightforward result.
\begin{lemma}\label{lem_equiv_opt_01}
	The optimization problem \eqref{opt_MKL_01} is equivalent to the following one:
	\begin{subequations}\label{opt_MKL_02}
		\begin{align}
			& \underset{\substack{f\in\Hbb,\ \db\in\Rbb^L\\ b\in\Rbb,\ \ub\in \Rbb^m}}{\min}
			& & \frac{1}{2} \|f\|^2_{\Hbb} + C\|\ub_+\|_0 \\
			& \quad\ \text{s.t.}
			& & \eqref{d_l_simplex_01}\ \text{and}\ \eqref{d_l_simplex_02} \nonumber \\
			& & & u_i + y_i(f(\xb_i) +b) =1,\ i\in\Nbb_m 
			\label{equal_constraint_u}
		\end{align}
	\end{subequations}
	in the sense that they have the same set of minimizers once we introduce the same equality constraint for $\ub$ in \eqref{opt_MKL_01} and identify $f=\sum_{\ell} f_\ell$.
\end{lemma}
\begin{proof}
	The proof relies on the relation \eqref{sqr_norm_direct_sum}. In view of that, we can rewrite the objective function of \eqref{opt_MKL_02} as 
	\begin{equation}\label{inner_min}
		\begin{aligned}
			\min \left\{ \frac{1}{2}\sum_{\ell} \frac{1}{d_\ell} \|f_\ell\|^2_{\Hbb_\ell} +C\|\ub_+\|_0 : f=\sum_{\ell} f_\ell \ \text{such that}\ f_\ell\in\Hbb_\ell' \right\}
		\end{aligned}
	\end{equation}
	where the $\ell_0$-norm term can be seen as a constant with respect to this inner minimization. 
	
	    Now suppose that the minimizer of \eqref{opt_MKL_02} is $(f^*, \db^*, b^*, \ub^*)$. We will show that it is also a minimizer of \eqref{opt_MKL_01}. The converse can be handled similarly and is omitted. 
	    Let the additive decomposition of $f$ that achieves the minimum in \eqref{inner_min} be $\sum_{\ell}\hat f_\ell$. Then subject to the feasibility conditions, we have
	    \begin{equation*}
			    		\frac{1}{2}\sum_{\ell} \frac{1}{d^*_\ell} \|\hat{f}^*_\ell\|^2_{\Hbb_\ell} +C\|\ub^*_+\|_0  \leq \frac{1}{2} \|f\|^2_{\Hbb} + C\|\ub_+\|_0 
			    		  \leq \frac{1}{2}\sum_{\ell} \frac{1}{d_\ell} \|f_\ell\|^2_{\Hbb_\ell} +C\|\ub_+\|_0,
		    \end{equation*}
		which means that $(\fb^*, \db^*, b^*)$ is a minimizer of \eqref{opt_MKL_01}.
\end{proof}

\section{Optimality Theory}\label{sec:optimality}

In this section, we give some theoretical results on the existence of an optimal solution to \eqref{opt_MKL_02} and equivalently, to \eqref{opt_MKL_01}, and some KKT-like first-order optimality conditions. 
Our standing assumption is that each kernel matrix  $K_\ell$ for the RKHS $\Hbb_\ell$ is positive definite as e.g., in the case of Gaussian kernels with different hyperparameters. We state this below formally.

\begin{assumption}\label{assump_posi_def_K_l}
	Given the data points $\{\xb_i : i\in \Nbb_m\}$, each $m\times m$ kernel matrix $K_\ell$, whose $(i, j)$ entry is $\kappa_\ell(\xb_i, \xb_j)$, is positive definite for $\ell\in\Nbb_L$.
\end{assumption}

The main results are given in the next two subsections.

\subsection{Existence of a Minimizer}\label{subsec:exist}

\begin{theorem}\label{thm_exist}
	Assume that the intercept $b$ takes value from the closed interval $\Ical:=[-M, M]$ where $M>0$ is a sufficiently large number. Then the optimization problem \eqref{opt_MKL_02} has a global minimizer and the set of all global minimizers is bounded.
\end{theorem}

\begin{proof}
	It is not difficult to argue that the problem \eqref{opt_MKL_01} is equivalent to 
	\begin{equation}\label{opt_MKL_02.1}
	\begin{aligned}
	& \underset{\db\in\Rbb^L}{\min} 
	& & \left\{ \underset{\substack{f\in\Hbb\\ b\in\Rbb}}{\min}\ \frac{1}{2} \|f\|^2_{\Hbb} + C \sum_{i} \|\left[1-y_i(f(\xb_i) +b)\right]_+\|_0 \right\} \\
	& \text{s.t.}
	& & \eqref{d_l_simplex_01}\ \text{and}\ \eqref{d_l_simplex_02}.
	\end{aligned}
	\end{equation}
	The inner minimization problem can be viewed as unconstrained once a feasible $\db$ is fixed. Hence we can	employ the {semiparametric representer theorem}  to conclude that the optimal $f$ of the inner optimization has the form 
	\begin{equation}\label{multi-kernel_representer}
	\begin{aligned}
	f(\xb) & = \sum_{\ell} f_\ell(\xb)  
	= \sum_{\ell} d_\ell \sum_{i} w_{i} \kappa_\ell(\xb,\xb_i),
	\end{aligned}
	\end{equation}
	where $f_\ell$ belongs to the RKHS $\Hbb_\ell'$ with a kernel $d_\ell \kappa_\ell(\cdot,\cdot)$ for each $\ell\in\Nbb_L$, and the parameter vector $\wb=[w_1,\dots,w_m]\in\Rbb^m$ .
	In view of this and \eqref{sqr_norm_direct_sum}, the optimization problem \eqref{opt_MKL_02.1} is further equivalent to
	\begin{equation}\label{opt_MKL_02.22}
	\begin{aligned}
	& \underset{\substack{\wb\in\Rbb^m,\ \db\in\Rbb^L \\ b\in\Rbb}}{\min} 
	& & \frac{1}{2} \wb^\top \Kcal(\db) \wb + C \|(\oneb - \Acal(\db)\wb -b\yb)_+\|_0 \\
	& \qquad\text{s.t.}
	& & \eqref{d_l_simplex_01}\ \text{and}\ \eqref{d_l_simplex_02},
	\end{aligned}
	\end{equation}
	where 
	$\Kcal(\db) := \sum_{\ell} d_\ell K_\ell$ with $K_\ell$ the kernel matrix corresponding to $\kappa_\ell(\cdot, \cdot)$, and 
	$\Acal(\db):=D_{\yb} \Kcal(\db)$.
	Let us write the objective function as	
	\begin{equation}\label{objective_finite_dim}
	J(\wb, \db, b) := \frac{1}{2}\textbf{w}^\top \Kcal(\db)\textbf{w} + C\sum_{i} \left\| \left[(\oneb - \Acal(\db)\wb -b\yb)_i \right]_+\right\|_0.
	\end{equation}
	It is obvious that the minimum
	value of $J(\wb, \db, b)$ is no larger than $J(\zerob, \oneb/L, 0)= Cm$, where $m$ is the maximum value of the $\ell_0$ norm for vectors in $\mathbb{R}^m$. We can then consider the optimization
	problem \eqref{opt_MKL_02.22} on the nonempty sublevel set
	\begin{equation}
		S:=\{(\wb, \db, b)\in\mathbb{R}^m\times\mathcal{I}\times\mathbb{R}^L : J(\wb, \db, b)\le Cm\}.
	\end{equation}
	We will prove that $S$ is a compact set. Due to the finite dimensionality, it amounts to showing that $S$ is closed
	and bounded. To this end, notice that the second term in the objective function \eqref{objective_finite_dim}
	is lower-semicontinuous because so is the function $\|(\cdot)_+\|_0$ on $\mathbb{R}$ and the argument $(\oneb - \Acal(\db)\wb -b\yb)_i$ is smooth in $(\wb, \db, b)$. 
	We can now conclude that $J(\wb, \db, b)$ is also lower-semicontinuous
	since the first quadratic term is smooth. Consequently,
	the sublevel set $S$ is closed. The fact that $S$ is also
	bounded follows from the argument that if we allow
	$\|\wb\|\to\infty$, then
	\begin{equation}
		J(\wb, \db, b)\ge \frac{1}{2}\textbf{w}^\top \Kcal(\db) \textbf{w}\ge \frac{1}{2}\lambda_{\min}(\Kcal(\db))\|\textbf{w}\|^2 \to \infty,
	\end{equation}
	where $\lambda_{\min}(\Kcal(\db))
	>0$ is the smallest eigenvalues of $\Kcal(\db)$. Therefore, the sublevel set $S$ is compact and the existence of a global minimizer follows from the extreme
	value theorem of Weierstrass. The set of all global minimizers is a subset of $S$ and is automatically bounded.
\end{proof}

\subsection{Characterization of Global and Local Minimizers}\label{subsec:optim_cond}


For the sake of consistency, let us rewrite \eqref{opt_MKL_01} in the following way:
\begin{subequations}\label{inf_opt_MKL_01.1}
	\begin{align}
		& \underset{\substack{\fb\in \Fbb,\ \db\in\Rbb^L\\ b\in\Rbb,\ \ub\in\Rbb^m}}{\min}
		& & \frac{1}{2} \sum_{\ell} \frac{1}{d_\ell} \|f_\ell\|^2_{\Hbb_\ell} + C \|\ub_+\|_0 \label{obj_func_J} \\
		& \quad\ \text{s.t.}
		& & \eqref{d_l_simplex_01},\ \eqref{d_l_simplex_02},\ \text{and}\ \eqref{equal_constraint_u}, \nonumber 
	\end{align}
\end{subequations}
where the last equality constraint \eqref{equal_constraint_u} is obviously \emph{affine} in the ``variables'' $(\fb, b, \ub)$, and we have written $f=\sum_{\ell} f_\ell$ for simplicity.
Before stating the optimality conditions, we need a generalized definition of a stationary point in nonlinear programming.

\begin{definition}[P-stationary point of \eqref{inf_opt_MKL_01.1}]\label{def_P1_stationary}
	Fix a regularization parameter  $C>0$. We call $(\fb^*, \db^*, b^*, \ub^*)$ a proximal stationary $($abbreviated as P-stationary$)$ point of \eqref{inf_opt_MKL_01.1} if there exists a vector $(\thetab^*, \alpha^*, \lambdab^*) \in \Rbb^{L+1+m}$ and a number $\gamma>0$ such that
	\begin{subequations}\label{inf_P-stationary_cond}
		\begin{align}
			d_\ell^\ast & \ge 0, \ \ell\in\Nbb_L \label{inf_primal_constraint_01} \\
			\sum_{\ell} d_\ell^\ast & = 1 \label{inf_primal_constraint_02} \\
			u_i^* + y_i \left( f^*(\xb_i) +b^* \right) & =1,\ i\in\Nbb_m \label{inf_primal_constraint_03} \\ 
			\theta_\ell^\ast & \ge 0, \ \ell\in\Nbb_L \label{inf_dual_constraint} \\
			\theta_\ell^\ast d_\ell^\ast & = 0, \ \ell\in\Nbb_L \label{inf_complem_slackness} \\
			\forall \ell\in\Nbb_L,\quad \frac{1}{d_\ell^\ast}f_\ell^\ast(\cdot) & =-\sum_{i}\lambda_i^* y_i \kappa_\ell(\,\cdot\,, \xb_i) \label{inf_stationary_cond_Lagrangian_f} \\
			-\frac{1}{2(d_\ell^\ast)^2}\|f_\ell^\ast\|^2_{\Hbb_\ell}+\alpha^\ast-\theta_\ell^\ast & = 0, \ \ell\in\Nbb_L \label{inf_stationary_cond_Lagrangian_d} \\
			\yb^\top \lambdab^* & = 0 \label{inf_stationary_cond_Lagrangian_b} \\
			\prox_{\gamma C \|(\cdot)_+\|_0} (\ub^\ast-\gamma \lambdab^\ast) & = \ub^\ast, \label{inf_stationary_cond_Lagrangian_u}
		\end{align}
	\end{subequations}
	where the proximal operator is defined as
	\begin{equation}\label{prox_def}
	\prox_{\gamma C\stepfunc} (\zb) := \underset{\vb\in\Rbb^m}{\argmin} \quad C\|\vb_+\|_0 + \frac{1}{2\gamma} \|\vb-\zb\|^2.
	\end{equation}
	\end{definition}

	It can be shown \citep{wang2021support} that the scalar version of the proximal operator above has a closed-form solution
\begin{equation}\label{prox_op_scalar}
	\prox_{\gamma C\stepfunc} (z) =
	\left\{\begin{aligned}
		0, & \quad 0<z\leq \sqrt{2\gamma C} \\
		z, & \quad z>\sqrt{2\gamma C} \ \text{or} \ z\leq 0,
	\end{aligned}\right.
\end{equation}
see also Fig.~\ref{fig:prox_operat}.
The vector version of the the proximal operator 
 is evaluated by a componentwise application of \eqref{prox_op_scalar}, that is,
\begin{equation}\label{prox_op_vec}
	[\prox_{\gamma C\stepfunc} (\zb)]_i = \prox_{\gamma C\stepfunc} (z_i)
\end{equation}
for $\zb\in\Rbb^m$
because the objective function on the right-hand side of \eqref{prox_def} admits an additive componentwise decomposition.
Formula \eqref{prox_op_vec} is called ``$L_{0/1}$ proximal operator'' in \cite{wang2021support}.
\begin{figure}
	\centering
	\begin{tikzpicture}
		\draw[line width=0.5pt][->](-1.8,0)--(2,0)node[left,below,font=\tiny]{$z$};
		\draw[line width=0.5pt][->](0,-1.5)--(0,2);
		\node[below,font=\tiny] at (0.09,0.1){0};
		\draw[color=red,thick,smooth][-](-1.2,-1.2)--(0,0);
		\draw[color=red,thick,smooth][-](0,0)--(0.8,0);
		\draw[color=red,fill=red,smooth](0.8,0)circle(0.03);
		\node[left,below,font=\tiny]at(0.8,0){$_{\sqrt{2\gamma C}}$};
		\draw[color=red][dashed] (0.8,0)--(0.8,0.8); 
		\draw[color=red,thick,smooth][-](0.8,0.8)--(1.8,1.8);
		\node[above,font=\tiny] at(1.2,1.7) {$\quad\mathrm{prox}_{\gamma C\|(\cdot)_+\|_0}(z)$};
	\end{tikzpicture}
	\caption{The $L_{0/1}$ proximal operator on the real line.}
	\label{fig:prox_operat}
\end{figure}
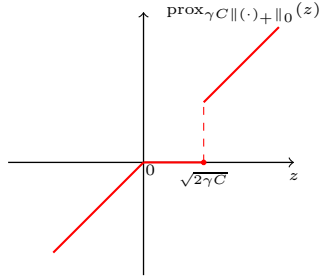

The elements of $(\thetab^*, \alpha^*, \lambdab^*)$ in Definition~\ref{def_P1_stationary} can be understood as
\emph{Lagrange multipliers} since they correspond to similar quantities in a smooth SVM problem 
\citep{cortes1995support}. However, here a dual problem seems difficult to derive because of the nonsmooth nonconvex function $\|(\cdot)_+\|_0$.
The set of equations and inequalities \eqref{inf_P-stationary_cond} are interpreted as \emph{KKT-like} optimality conditions for the optimization problem \eqref{inf_opt_MKL_01.1}, where \eqref{inf_primal_constraint_01}, \eqref{inf_primal_constraint_02}, and \eqref{inf_primal_constraint_03} are the primal constraints, \eqref{inf_dual_constraint} the dual constraints, \eqref{inf_complem_slackness} the complementary slackness, and \eqref{inf_stationary_cond_Lagrangian_f}, \eqref{inf_stationary_cond_Lagrangian_d}, \eqref{inf_stationary_cond_Lagrangian_b}, and \eqref{inf_stationary_cond_Lagrangian_u} the stationarity conditions of the Lagrangian with respect to the primal variables. 
We must point out that the nonsmoothness of the problem \eqref{inf_opt_MKL_01.1} is only found in the term $\|\ub_+\|_0$, and the corresponding stationarity condition \eqref{inf_stationary_cond_Lagrangian_u} 
is given in terms of the proximal operator.

The following theorem connects the optimality conditions for \eqref{inf_opt_MKL_01.1} to P-stationary points.

\begin{theorem}\label{thm_optimality}
	The global and local minimizers of \eqref{inf_opt_MKL_01.1} admit the following characterizations:	
	\begin{enumerate}[label=(\arabic*)]
		\item A global minimizer is a P-stationary point with $0<\gamma< C_1 $,
		where the positive number
		\begin{equation*}
			C_1 = \min\left\{\lambda_{\min}(\Kcal(\db)) : \db\ \text{satisfies} \ \eqref{d_l_simplex_01}\ \text{and}\ \eqref{d_l_simplex_02}\right\}
		\end{equation*}
		in which $\lambda_{\min}(\cdot)$ denotes the smallest eigenvalue of a matrix.
		
		\item Any P-stationary point $($with $\gamma>0$$)$ is also a local minimizer of \eqref{inf_opt_MKL_01.1}.
	\end{enumerate}
\end{theorem}

\begin{proof}
	To prove the first assertion, let $(\fb^*, \db^*, b^*, \ub^*)$ be a global minimizer of \eqref{inf_opt_MKL_01.1}. If $\ub^*$ is held fixed, then trivially, $(\fb^*, \db^*, b^*)$ must be a minimizer of the corresponding optimization problem with respect to the remaining variables $(\fb, \db, b)$. In other words, we have
	\begin{equation}\label{opt_MKL_01.1_fix_u}
	\begin{aligned}
	(\fb^*, \db^*, b^*) = & \underset{\substack{\fb\in \Fbb,\ \db\in\Rbb^L\\ b\in\Rbb}}{\argmin}
	& & \frac{1}{2} \sum_{\ell} \frac{1}{d_\ell} \|f_\ell\|^2_{\Hbb_\ell} + C \|\ub^*_+\|_0 \\
	& \ \quad\text{s.t.}
	& & \eqref{d_l_simplex_01}\ \text{and}\ \eqref{d_l_simplex_02} \\
	& & & u^*_i + y_i \left( f(\xb_i) +b \right) =1\ \forall i
	\end{aligned}
	\end{equation}
	Since the term $C \|\ub^*_+\|_0$ now becomes a fixed constant, the above problem turns into a \emph{smooth convex} optimization problem, thanks to the result in  \citet[Appendix~A.1]{rakotomamonjy2008simplemkl}. Moreover, Slater's condition clearly holds so we have \emph{strong duality}. Consequently, conditions \eqref{inf_primal_constraint_01} through \eqref{inf_stationary_cond_Lagrangian_b} hold as KKT conditions for \eqref{opt_MKL_01.1_fix_u}.

	It now remains to show the last condition \eqref{inf_stationary_cond_Lagrangian_u}. This time we fix $\db^*$ and consider the minimization with respect to the other three variables. More precisely, we shall consider the  formulation equivalent to \eqref{opt_MKL_02.22} with the additional variable $\ub$:
	\begin{equation}\label{opt_MKL_02.3}
		\begin{aligned}
			(\wb^*, b^*, \ub^*) = & \underset{\substack{\wb\in\Rbb^m,\ b\in\Rbb \\ \ub\in\Rbb^m}}{\argmin} 
			& & \frac{1}{2} \wb^\top \Kcal(\db^*) \wb + C \|\ub_+\|_0 \\
			& \quad\ \text{s.t.}
			& & \ub + \Acal(\db^*)\wb + b\yb = \oneb.
		\end{aligned}
	\end{equation}
	Since the matrix $\Acal(\db^*)=D_{\yb}\Kcal(\db^*)$ is invertible by Assumption~\ref{assump_posi_def_K_l}, we can write $\wb=\Acal(\db^*)^{-1} (\oneb-\ub-b\yb)$ and convert the optimization problem to another unconstrained version:
	\begin{equation}\label{optim_unconstrain_u_b}
		(\ub^*, b^*) = \underset{\substack{\ub\in\Rbb^m, \, b\in\Rbb}}{\argmin}\ \  g(\ub,b) + C \|\ub_+\|_0
	\end{equation}
	where $g(\ub,b):=  \frac{1}{2} (\oneb-\ub-b\yb)^\top \Kcal_1(\db^*)^{-1} (\oneb-\ub-b\yb)$ is a new quadratic term such that the matrix $\Kcal_1(\db) := D_{\yb} \Kcal(\db) D_{\yb}$. The following computation is straightforward: 
	\begin{equation}\label{grad_Hess_g}
		\begin{aligned}
			\nabla g(\ub,b) & = \bmat \nabla_{\ub}\, g(\ub,b) \\ \nabla_b\, g(\ub,b)\emat = - \bmat I \\ \yb^\top \emat \Kcal_1(\db^*)^{-1} (\oneb-\ub-b\yb), \\
			\nabla^2 g(\ub,b) & = \bmat I \\ \yb^\top \emat \Kcal_1^{-1}(\db^*) \bmat I & \yb \emat.
		\end{aligned}
	\end{equation}
	Next define 
	\begin{equation}\label{def_lambda*}
		\begin{aligned}
			\lambdab^* & := -\Kcal_1(\db^*)^{-1} (\oneb-\ub^*-b^*\yb) 
			 = -\Kcal_1(\db^*)^{-1} \Acal(\db^*) \wb^* = -D_{\yb} \wb^*.
		\end{aligned}
	\end{equation}
	Take an arbitrary number $\gamma\in(0, C_1)$, and define a vector $\zb:=\prox_{\gamma C\stepfunc} (\ub^* - \gamma\lambdab^*)$. We need to prove $\zb=\ub^*$ which is the last equation in \eqref{inf_P-stationary_cond}.
	To this end, we shall emphasize three points:
	
	\begin{itemize}
		\item[(i)] By global optimality of $(\ub^*,b^*)$, we have 
		\begin{equation}
			g(\ub^*,b^*) +C\|\ub^*_+\|_0 \leq g(\zb,b^*) +C\|\zb_+\|_0.
		\end{equation}
		\item[(ii)] Using second-order Taylor expansion for $g(\ub,b)$, it is not difficult to show 
		\begin{equation}
			g(\zb,b^*) - g(\ub^*,b^*) \leq (\lambdab^*)^\top (\zb-\ub^*) + \frac{\lambda_1}{2} \|\zb-\ub^*\|^2
		\end{equation}
		where $\lambda_1:=\lambda_{\max}(\Kcal_1(\db^*)^{-1})=1/\lambda_{\min}(\Kcal(\db^*))$.
		\item[(iii)] By the definition of the proximal operator \eqref{prox_def}, we have
		\begin{equation}
			\begin{aligned}
				& C \|\zb_+\|_0 + \frac{1}{2\gamma} \|\zb-(\ub^*-\gamma\lambdab^*)\|^2 
				\leq  C \|\ub^*_+\|_0 + \frac{1}{2\gamma} \|\gamma\lambdab^*\|^2 = C \|\ub^*_+\|_0 + \frac{\gamma}{2} \|\lambdab^*\|^2.
			\end{aligned}
		\end{equation}
	\end{itemize}
	Combining the above points via a chain of inequalities similar to \citet[Eq.~(S12), supplementary material]{wang2021support}, one can arrive at
	\begin{equation}
		0\leq \frac{\lambda_1 - 1/\gamma}{2} \|\zb-\ub^*\|^2 \leq 0
	\end{equation}
	where the constant term in the middle is negative since we have chosen $0<\gamma<C_1\leq \lambda_{\min}(\Kcal(\db^*)=1/\lambda_1$. Therefore, we conclude that $\zb=\ub^*$.

To prove the second part of the assertion, let us introduce the symbol $\phi=(\fb, \db,b, \textbf{u})$ for convenience. Suppose now that we have a P-stationary point $\phi^*=(\fb^*, \db^*,b^*, \ub^*)$ with an associated vector $(\thetab^*, \alpha^*, \lambdab^*)$ and $\gamma>0$.
Moreover, let $U(\phi^*, \delta_1)$ be a sufficiently small neighborhood of $\phi^*$ of radius $\delta_1$ such that for any $\phi=(\fb, \db,b, \textbf{u})\in U(\phi^*,\delta_1)$ we have $u^*_i\neq 0\implies u_i\neq 0$. That is to say: a small perturbation of $\textbf{u}^*$ does not change the signs of its nonzero components. As a consequence, the relation
\begin{equation}\label{local_u}
	\|\textbf{u}_+\|_0\geq\|\textbf{u}^*_+\|_0
\end{equation}
holds in that neighborhood. Next, notice that the first term $\frac{1}{2}\sum_{\ell} \frac{1}{d_\ell}\|f_\ell\|^2_{\Hbb_\ell}$ in  the objective function \eqref{obj_func_J} is smooth  in $(\fb, \textbf{d})$, which implies that it is locally Lipschitz continuous. Consequently, there 
exists a neighborhood  $U(\phi^*, \delta_2)$ such that for any $\phi\in U(\phi^*, \delta_2)$, we have $\left|\frac{1}{2}\sum_{\ell}\frac{1}{d_\ell}\|f_\ell\|^2_{\Hbb_\ell}-\frac{1}{2}\sum_{\ell}\frac{1}{d_\ell^\ast}\|f_\ell^\ast\|^2_{\Hbb_\ell}\right|\le C$ which further implies that
\begin{equation}\label{local_inequal_1}
	\frac{1}{2}\sum_{\ell}\frac{1}{d_\ell^\ast}\|f_\ell^\ast\|^2_{\Hbb_\ell}-C\le \frac{1}{2}\sum_{\ell} \frac{1}{d_\ell}\|f_\ell\|^2_{\Hbb_\ell}.
\end{equation}
Now, take $\delta=\min\left\{\delta_1,\delta_2\right\}>0$ and consider the feasible region
\begin{equation}
	\begin{aligned}
		\Theta := \{\phi = (\fb, \db,b, \textbf{u}) : 
		\eqref{d_l_simplex_01}, \eqref{d_l_simplex_02},\ \text{and} \ \eqref{equal_constraint_u}\ \text{hold}
		 \}
	\end{aligned}	
\end{equation}
of the problem \eqref{inf_opt_MKL_01.1}. We will show that the P-stationary point $\phi^\ast$ is locally optimal in $\Theta\cap U(\phi^*,\delta)$, that is, $\phi\in\Theta\cap U(\phi^*,\delta)$ implies the inequality
\begin{equation}\label{local_optim}
	\frac{1}{2}\sum_{\ell} \frac{1}{d_\ell^\ast}\|f_\ell^\ast\|^2_{\Hbb_\ell}+C\left\|\textbf{u}^\ast_+\right\|_0 \le \frac{1}{2}\sum_{\ell}\frac{1}{d_\ell}\|f_\ell\|^2_{\Hbb_\ell}+C\left\|\textbf{u}_+\right\|_0.
\end{equation}
For this purpose, let $\Gamma_*:=\{i : u^*_i=0\}$ and $\overline{\Gamma}_* := \mathbb{N}_m\backslash \Gamma_*$.
Then by 
\eqref{inf_stationary_cond_Lagrangian_u} and the evaluation formulas for the proximal operator (see \eqref{prox_op_scalar}and \eqref{prox_op_vec}), we have the relation:
\begin{equation}\label{relation_Gamma_star}
	\begin{cases}
		i\in\Gamma_*,\ u_i^*=0,\ -\sqrt{2C/\gamma} \leq \lambda^*_i <0, \\
		i\in\overline{\Gamma}_*, \ u_i^*\neq 0, \ \lambda^*_i=0.
	\end{cases}
\end{equation}
Consider the subset $\Theta_1:= \Theta \cap \{\phi : u_i\leq 0 \ \forall i\in\Gamma_*\}$ of $\Theta$, and the partition $\Theta=\Theta_1\cup (\Theta\backslash\Theta_1)$. We will split the discussion into two cases given such a partition of the feasible region.

\emph{Case 1}: $\phi\in\Theta_1\cap U(\phi^*,\delta)$. We emphasize two conditions:
\begin{equation}\label{two_conds}
	u_i\le 0 \ \ \text{for}\ \ i\in \Gamma_* \quad  \text{and} \quad u_i+y_i \left(\sum_{\ell} f_\ell(\textbf{x}_i)+b\right) =1.
\end{equation}
The following chain of inequalities hold:
\begin{subequations}\label{local_f}
	\begin{align}
			&\frac{1}{2}\sum_{\ell}\frac{1}{d_\ell}\|f_\ell\|^2_{\Hbb_\ell}-\frac{1}{2}\sum_{\ell}\frac{1}{d_\ell^\ast}\|f_\ell^\ast\|^2_{\Hbb_\ell} \nonumber \\
			\ge &\, \frac{1}{2}\sum_{\ell} \left[\frac{1}{d_\ell^\ast} \langle 2f_\ell^\ast,f_\ell-f_\ell^\ast\rangle - \frac{d_\ell-d_\ell^\ast}{(d_\ell^\ast)^2} \|f_\ell^\ast\|^2_{\Hbb_\ell} \right] \label{norm_inequal_01} \\
			= & -\sum_{i}\lambda_i^\ast y_i \left[\sum_{\ell}f_\ell(\textbf{x}_i)-\sum_{\ell}f_\ell^\ast(\textbf{x}_i)\right]
			+\sum_{\ell}(\theta_\ell^\ast-\alpha^\ast)(d_\ell-d_\ell^\ast) \label{norm_inequal_02} \\
			= & \sum_{i}\lambda_i^\ast(u_i-u_i^\ast)+\sum_{\ell}\theta_\ell^\ast d_\ell-\sum_{\ell}\theta_\ell^\ast d_\ell^\ast \label{norm_inequal_03} \\
			\ge & \sum_{i}\lambda_i^\ast(u_i-u_i^\ast)
			=\sum_{i\in \Gamma_*}\lambda^\ast_{i} u_{i}\ge 0, \label{norm_inequal_04}
	\end{align} 
\end{subequations} 
where, 
\begin{itemize}
	\item \eqref{norm_inequal_01} is the first order condition for convexity of each term $\frac{1}{d_\ell}\|f_\ell\|^2_{\Hbb_\ell}$,
	\item \eqref{norm_inequal_02} comes from the conditions \eqref{inf_stationary_cond_Lagrangian_f} and \eqref{inf_stationary_cond_Lagrangian_d} for the P-stationary point and the kernel trick, 
	\item \eqref{norm_inequal_03} results from the feasibility condition on the right of \eqref{two_conds}, the fact that $\alpha^*$ is independent of the dummy variable $\ell$, and the condition \eqref{d_l_simplex_02},
	\item and \eqref{norm_inequal_04} is a consequence of \eqref{relation_Gamma_star} and the condition on the left of \eqref{two_conds}.
\end{itemize}
Combining \eqref{local_u} and \eqref{local_f}, we obtain \eqref{local_optim}.

\emph{Case 2}: $\phi\in (\Theta\backslash\Theta_1)\cap U(\phi^*,\delta)$. Note that $\phi\in\Theta\backslash\Theta_1$ means that there exists some $i\in\Gamma_*$ such that $u_i>0$ while $u_i^*=0$ (by the definition of $\Gamma_*$). Then we have $\|\textbf{u}_+\|_0\geq\|\textbf{u}^*_+\|_0+1$. Combining this with \eqref{local_inequal_1}, we have
\begin{equation}
	\begin{aligned}
		\frac{1}{2}\sum_{\ell} \frac{1}{d_\ell^\ast}\|f_\ell^\ast\|^2_{\Hbb_\ell}+C\left\|\textbf{u}^\ast_+\right\|_0 & 
		\le\ \frac{1}{2}\sum_{\ell}\frac{1}{d_\ell^\ast}\|f_\ell^\ast\|^2_{\Hbb_\ell}+C\left\|\textbf{u}_+\right\|_0-C \\
		 & \le \frac{1}{2}\sum_{\ell} \frac{1}{d_\ell}\|f_\ell\|^2_{\Hbb_\ell}+C\left\|\textbf{u}_+\right\|_0,
	\end{aligned}
\end{equation}
as desired. This completes the proof of local optimality of $\phi^*$. 
\end{proof}

\section{Algorithm Design}\label{sec:algorithm}

In this section, we take advantages of the ADMM \citep{boyd2011distributed} and working sets (active sets) to devise a first-order algorithm for our MKL-$L_{0/1}$-SVM optimization problem. Before describing the algorithmic details, we must point out that the finite-dimensional formulation \eqref{opt_MKL_02.22} of the problem is more suitable for computation than the (generally) infinite-dimensional formulation \eqref{inf_opt_MKL_01.1}.
For this reason, we work on the former formulation in this section. The theory developed in Section~\ref{sec:optimality}, in particular, the notion of a P-stationary point, holds for the finite-dimensional problem modulo suitable adaptation which is done next.

First, let us rewrite \eqref{opt_MKL_02.22} as
\begin{subequations}\label{opt_MKL_01.1}
	\begin{align}
	& \underset{\substack{\wb\in \Rbb^m,\ \db\in\Rbb^L\\ b\in\Rbb,\ \ub\in\Rbb^m}}{\min}
	& &\frac{1}{2} \wb^\top \Kcal(\db) \wb + C \|\ub_+\|_0\\
	& \quad \quad \text{s.t.}
	& & \eqref{d_l_simplex_01}\ \text{and}\ \eqref{d_l_simplex_02} \nonumber \\
	& & & \ub+\Acal(\db)\wb +b\yb=\oneb
	\end{align}
\end{subequations}



\begin{definition}[P-stationary point of \eqref{opt_MKL_01.1}]\label{def_P2_stationary}
	Fix a regularization parameter $C>0$. We call $(\wb^*, \db^*, b^*, \ub^*)$ a P-stationary point of \eqref{opt_MKL_01.1} if there exists a vector $(\thetab^*, \alpha^*, \lambdab^*) \in \Rbb^{L+1+m}$ and a number $\gamma>0$ such that the conditions \eqref{inf_primal_constraint_01}, \eqref{inf_primal_constraint_02}, \eqref{inf_dual_constraint}, \eqref{inf_complem_slackness}, \eqref{inf_stationary_cond_Lagrangian_b}, \eqref{inf_stationary_cond_Lagrangian_u}, and
	\begin{subequations}\label{P-stationary_cond}
		\begin{align}
		\ub^*+\Acal(\db^*)\wb^* +b^*\yb&=\oneb, \label{primal_constraint_03} \\ 
		\wb^*+ D_{\yb}\lambdab^*&=\zerob,  \label{stationary_cond_Lagrangian_w}\\
		-\frac{1}{2}(\wb^*)^\top K_\ell \wb^*+\alpha^*-\theta_\ell^* & =0, \ \ell\in\Nbb_L \label{stationary_cond_Lagrangian_d}
		\end{align}
	\end{subequations}
hold.
\end{definition}

	One can show that Definition~\ref{def_P2_stationary} above and the previous Definition~\ref{def_P1_stationary} are equivalent, as stated in the next proposition.
	
	\begin{proposition}
		Given a P-stationary point $(\wb^*, \db^*, b^*, \ub^*)$ in the sense of Definition~\ref{def_P2_stationary}, there corresponds a P-stationary point $(\fb^*, \db^*, b^*, \ub^*)$ in the sense of Definition~\ref{def_P1_stationary}, and vice versa.		
	\end{proposition}

\begin{proof}
	We shall first show that Definition~\ref{def_P2_stationary}$\implies$Definition~\ref{def_P1_stationary}\,.
	Given $(\wb^*,\db^*,b^*,\ub^*)$ and $(\thetab^*,\alpha^*,\lambdab^*)$ in Definition~\ref{def_P2_stationary},
	we construct 
	\begin{equation}\label{def_f_l*}
	{f_\ell^\ast(\cdot)} = d_\ell^\ast \sum_{i}w_i^\ast \kappa_\ell(\,\cdot\,, \xb_i)\quad \forall\ell\in\Nbb_L
	\end{equation}
	which will be used to check conditions in Definition~\ref{def_P1_stationary}.
	In view of the relation \eqref{stationary_cond_Lagrangian_w}, we immediately recover \eqref{inf_stationary_cond_Lagrangian_f}.
	Moreover, we have from \eqref{def_f_l*} that 
	$$\|f_\ell^\ast\|^2_{\Hbb_\ell} = {(d_\ell^\ast)^2}(\wb^\ast)^\top K_\ell \wb^\ast$$
	which implies the equivalence between
\eqref{inf_stationary_cond_Lagrangian_d} and \eqref{stationary_cond_Lagrangian_d}. 
	The condition \eqref{inf_primal_constraint_03} is  equivalent to \eqref{primal_constraint_03} because the latter is simply a vectorized notation for the former. To see this point, just notice that the vector of function values $[f^*(\xb_1), \cdots, f^*(\xb_m)]^\top$ is equal to $\Kcal(\db^*)\wb^*$ in view of \eqref{def_f_l*} and $f^*=\sum_{\ell} f^*_\ell$.
	
	
	 For the converse, i.e., Definition~\ref{def_P1_stationary}$\implies$Definition~\ref{def_P2_stationary}, we simply define $\wb^* := - D_{\yb} \lambdab^*$, which is equivalent to \eqref{stationary_cond_Lagrangian_w}, given the $\lambdab^*$ in the sense of Definition~\ref{def_P1_stationary}. Then the above proof holds almost verbatim.
\end{proof}

	\begin{remark}
		The linear relation \eqref{stationary_cond_Lagrangian_w} between $\lambdab^*$ and $\wb^*$ is consistent with \eqref{def_lambda*} and is analogous to the first formula in \citet[Eq.~(11)]{wang2021support}. 
	\end{remark}

Now, we aim to obtain via numerical computation a P-stationary point in the sense of Definition~\ref{def_P2_stationary} and hence  a local minimizer of \eqref{opt_MKL_01.1} by Theorem~\ref{thm_optimality}. 
Before going into algorithmic details, let us recognize a ``working set'' (with respect to the data) and its relation with \emph{support vectors} following the development of \citet[Subsec.~4.1]{wang2021support}, see also \citet[Sec.~V]{shi2023admm}.
Suppose that $(\wb^*, \db^*, b^*, \ub^*)$ is a P-stationary point of \eqref{opt_MKL_01.1}. Then according to Definition~\ref{def_P2_stationary}, there exists a Lagrange multiplier $\lambdab^*\in\Rbb^m$ and a scalar $\gamma>0$ such that \eqref{inf_stationary_cond_Lagrangian_u} holds. Define a set
\begin{equation}\label{work_set_T*}
	T_*:=\left\{i\in \mathbb{N}_m : u_i^*-\gamma\lambda_i^* \in (0,\sqrt{2\gamma C}\,]\right\},
\end{equation}
and its complement $\overline{T}_*:=\Nbb_m\backslash T_*$. For a vector $\zb\in\Rbb^m$ and an index set $T\subset\Nbb_m$ with cardinality $|T|$, we write $\zb_T$ for the $|T|$-dimensional subvector of $\zb$ whose components are indexed in $T$. Then after some straightforward derivation using \eqref{inf_stationary_cond_Lagrangian_u}, \eqref{prox_op_scalar} and \eqref{prox_op_vec}, we will obtain 
$\ub^*_{T_*}=\zerob$ and $\lambdab^*_{\overline{T}_*}=\zerob$, and the working set \eqref{work_set_T*} also admits the expression
\begin{equation}\label{work_set_T*_equiv}
	T_*=\left\{i\in \mathbb{N}_m : \lambda_i^* \in [ -\sqrt{2C/\gamma}, 0)\right\}.
\end{equation}
The formula tells us that the nonzero components of $\lambdab^*$ are indexed only in $T_*$ with values in the interval $[ -\sqrt{2C/\gamma}, 0)$. The following comments are immediate:


\begin{itemize}
	\item The vectors $\{\xb_i : i\in T_*\}$ corresponding to nonzero Lagrange multipliers $\{\lambda_i^*\}$ are called \emph{$L_{0/1}$-support vectors} in \cite{wang2021support}. 
	They play the same role as standard support vectors in \cite{cortes1995support}. 
	
	\item Furthermore, the condition \eqref{inf_primal_constraint_03} and $\ub^*_{T_*}=\zerob$ imply that 
	any $L_{0/1}$-support vector $\xb_i$ satisfies one of the functional equations
		$f^*(\xb) +b^* =\pm 1.$ 
	In the classic (linear) case, the two equations determine the
	\emph{support hyperplanes}. 
\end{itemize}


At this point, we are ready to give the framework of ADMM for the finite-dimensional optimization problem \eqref{opt_MKL_01.1}. In order to handle the inequality constraints \eqref{d_l_simplex_01}, we employ the indicator function (in the sense of Convex Analysis) 
\begin{equation}
	g(\zb) = \left\{\begin{aligned}
		0, & \quad \zb\in\Rbb_+^L \\
		+\infty, & \quad \zb\notin\Rbb_+^L
	\end{aligned}\right.
\end{equation}
of the nonnegative orthant $\Rbb_+^L$ and convert \eqref{opt_MKL_01.1} to the form: 
\begin{subequations}\label{ADMM-form}
	\begin{align}
		& \underset{\substack{\wb\in \Rbb^m,\ \db\in\Rbb^L\\ b\in\Rbb,\ \ub\in\Rbb^m,\ \zb\in\Rbb^L}}{\min}
		& & \frac{1}{2} \wb^\top \Kcal(\db) \wb + C \|\ub_+\|_0 +g(\zb)\\
		& \qquad\ \ \text{s.t.}
		& & \db-\zb=\textbf{0} \label{constaint_indica_z}\\
		& &  & \textbf{1}^\top\db=1       \\
		& & & \ub+\Acal(\db)\wb +b\yb=\oneb,
	\end{align}
\end{subequations}
see \citet[Section~5]{boyd2011distributed}.
The \emph{augmented Lagrangian} of the \eqref{ADMM-form} is given by 
\begin{equation}
	\begin{aligned}
		\Lcal_\rho(\wb,\db,b,\ub,\zb; \thetab,\alpha, \lambdab) = & \,\frac{1}{2} \wb^\top \Kcal(\db) \wb + C \|\ub_+\|_0 +g(\zb) 
		+ \thetab^\top (\db-\zb)+\frac{\rho_2}{2} \|\db-\zb\|^2 \\
		 & +\alpha\left(\oneb^\top\db-1\right)+\frac{\rho_3}{2}\left(\oneb^\top\db-1\right)^2 +\lambdab^\top\left[\ub+\Acal(\db)\wb +b\yb-\oneb\right]\\
		 &  +\frac{\rho_1}{2}\| \ub+\Acal(\db)\wb +b\yb-\oneb\|^2  \\
	\end{aligned}
\end{equation}
where $\thetab, \alpha, \lambdab$ are the Lagrange multipliers and each $\rho_j>0$, $j=1, 2, 3$, is a penalty parameter.
Given the $k$-th iterate $(\wb^k, \db^k, b^k, \ub^k, \zb^k; \thetab^k, \alpha^k, \lambdab^k)$, the algorithm to update each variable runs as follows:
\begin{subequations}\label{ADMM-main}
	\begin{align}
		\ub^{k+1} & =  \underset{\substack{\ub\in\Rbb^m}}{\argmin}\ \Lcal_{\rhob}(\wb^k,\db^k,b^k,\ub,\zb^k; \thetab^k,\alpha^k,\lambdab^k) \label{subprob_u} \\
		\wb^{k+1} & = \underset{\substack{\wb\in\Rbb^m}}{\argmin}\ \Lcal_{\rhob}(\wb,\db^k,b^k,\ub^{k+1},\zb^k; \thetab^k,\alpha^k,\lambdab^k) \label{subprob_w}\\ 
		b^{k+1} & = \underset{\substack{b\in\Rbb}}{\argmin}\ \Lcal_{\rhob}(\wb^{k+1},\db^k,b,\ub^{k+1},\zb^k; \thetab^k,\alpha^k,\lambdab^k) \label{subprob_b} \\
		\zb^{k+1} & = \underset{\substack{\zb\in\Rbb^L}}{\argmin}\ \Lcal_{\rhob}(\wb^{k+1},\db^k,b^{k+1},\ub^{k+1},\zb; \thetab^k,\alpha^k,\lambdab^k) \label{subprob_z} \\
		\db^{k+1} & = \underset{\substack{\db\in\Rbb^L}}{\argmin}\ \Lcal_{\rhob}(\wb^{k+1},\db,b^{k+1},\ub^{k+1},\zb^{k+1}; \thetab^k,\alpha^k,\lambdab^{k}) \label{subprob_d} \\
		\thetab^{k+1} & = \thetab^k+\rho_2(\db^{k+1}-\zb^{k+1}) \label{dual_update_theta} \\
		\alpha^{k+1} & = \alpha^k+\rho_3\left(\oneb^\top \db^{k+1}-1\right) \label{dual_update_alpha} \\
		\lambdab^{k+1} & = \lambdab^k+\rho_1 \left[\ub^{k+1}+\Acal(\db^{k+1})\wb^{k+1} +b^{k+1}\yb-\oneb\right]. \label{dual_update_lambda}
	\end{align}
\end{subequations}

Next, we describe how to solve each subproblem above.
\begin{enumerate}[label=(\arabic*)]

	\item \textbf{Updating} $\ub^{k+1}$. The $\ub$-subproblem in \eqref{subprob_u} is equivalent to the following problem
	\begin{equation}\label{update_u_prox}
		\begin{aligned}
			\ub^{k+1} & = \underset{\substack{\ub\in \Rbb^m}}{\argmin} \ \ C \|\ub_+\|_0+\lambdab^\top\ub 
			+\frac{\rho_1}{2}\|\ub+\Acal(\db^{k})\wb^{k} +b^{k}\yb-\oneb\|^2\\
			& = \underset{\substack{\ub\in \Rbb^m}}{\argmin} \ \ C \|(\ub)_+\|_0+\frac{\rho_1}{2}\|\ub-\textbf{s}^k\|^2\\
			&= \prox_{\frac{C}{\rho_1}\|(\cdot)_+\|_0}(\textbf{s}^k),
		\end{aligned}	
	\end{equation}
	where
	\begin{equation}\label{vec_s^k}
		\textbf{s}^k=\oneb-\Acal(\db^{k})\wb^{k} -b^{k}\yb-\lambdab^k/\rho_1.
	\end{equation}
	Define a working set $T_k$ at the $k$-th step  by
	\begin{equation}\label{work_set_T_k}
		T_k:= \left\{i\in \mathbb{N}_m:s_i^{k}\in(0,\sqrt{2C/\rho_1})\right\}.
	\end{equation}
	Then  \eqref{update_u_prox} can equivalently be written as
	\begin{equation}\label{update_u}
		\ub_{T_k}^{k+1}=\textbf{0},\quad\ub_{\overline{T_k}}^{k+1}=\sbf^k_{\overline{T}_k}.
	\end{equation}

\item \textbf{Updating} $\wb^{k+1}$. The $\wb$-subproblem in \eqref{subprob_w} is
\begin{equation}\label{w-solution}
	\begin{aligned}
		\wb^{k+1}=\underset{\substack{\wb\in\Rbb^m}}{\argmin}\ \
		&\frac{1}{2} \wb^\top \Kcal(\db^k) \wb  
		+\lambdab^\top(\ub^{k+1}+\Acal(\db^{k})\wb +b^{k}\yb-\oneb)\\
		&+\frac{\rho_1}{2}\| \ub^{k+1}+\Acal(\db^{k})\wb +b^{k}\yb-\oneb\|^2.
	\end{aligned}
\end{equation}
This is a quadratic minimization problem, and we only need to find a solution to the stationary-point equation
\begin{equation}\label{w_limit}
		\Kcal(\db^{k})\wb+\Acal(\db^{k})^\top\lambdab^{k}
		+\rho_1\Acal(\db^k)^\top \left[\ub^{k+1}+\Acal(\db^k)\wb+b^k\yb-\oneb\right] = \textbf{0}
\end{equation}
which, after a rearrangement of terms, is a linear equation
\begin{equation}\label{stationa_eqn_w}
				\left[I+\rho_1\Kcal(\db^k)\right]\wb=
				-D_{\yb}\left[\lambdab^k+\rho_1(\ub^{k+1}+b^k\yb-\oneb)\right].
		\end{equation}
The coefficient matrix is obviously positive definite and hence a unique solution $\wb^{k+1}$ exists.

		\item \textbf{Updating} $b^{k+1}$. The $b$-subproblem in \eqref{subprob_b} can be written as
		\begin{equation}
			\begin{aligned}
				b^{k+1}&=\underset{\substack{b\in\Rbb}}{\argmin}\ 	(\lambdab^k)^\top b\yb
				+\frac{\rho_1}{2}\| \ub^{k+1}+\Acal(\db^{k})\wb^{k+1} +b\yb-\oneb\|^2.
			\end{aligned}
		\end{equation}
		The stationary-point equation is
		\begin{equation}\label{stationa_eqn_b}
				\yb^\top\lambdab^k 
				+\rho_1 \yb^\top(\ub^{k+1}+\Acal(\db^{k})\wb^{k+1} +b\yb-\oneb) = 0,
		\end{equation}
		and a solution is given by
		\begin{equation}\label{update_b}
			\begin{aligned}
				b^{k+1}&=-\frac{\yb^\top \left[\lambdab^k+\rho_1(\ub^{k+1}+\Acal(\db^k)\wb^{k+1}-\oneb)\right]}{m\rho_1}.
			\end{aligned}
		\end{equation}

		\item \textbf{Updating} $\zb^{k+1}$. The subproblem for $\zb^{k+1}$ in \eqref{subprob_z} also admits a separation of variables and we can carry out the update for each component as follows:
		\begin{equation}
			\begin{aligned}
				z_\ell^{k+1}&=\underset{\substack{z_\ell\in\Rbb}}{\argmin}\ \ g(\zb)-(\thetab^k)^\top\zb+\frac{\rho_2}{2}\|\db^k-\zb\|^2\\
				&=\underset{\substack{z_\ell\in\Rbb}}{\argmin}\ \ g(z_\ell)- \theta^k_\ell z_\ell+\frac{\rho_2}{2} (d^k_\ell-z_\ell)^2\\
				&=(d_\ell^{k}+{\theta_{\ell}^k}/{\rho_2})_+
			\end{aligned}
		\end{equation}
		where the function $(\cdot)_+$ takes the positive part of the argument. Inspired by this expression, we can define another working set 
		\begin{equation}\label{work_set_S_k}
			S_k := \left\{\ell\in\Nbb_L : d_\ell^{k}+{\theta_{\ell}^k}/{\rho_2}>0 \right\}
		\end{equation}
		for the selection of kernels and $\overline{S}_k=\Nbb_L\backslash S_k$. Then an equivalent update formula for $\zb$ is
		\begin{equation}\label{update_z}
			\zb^{k+1}_{S_k} = (\db^k+\thetab^k/\rho_2)_{S_k}, \quad \zb^{k+1}_{\overline{S}_k} = \zerob.
		\end{equation}
		Notice that this working set is less complicated than $T_k$ in \eqref{work_set_T_k} since the function $(\cdot)_+$ is continuous, unlike the proximal mapping \eqref{prox_op_scalar}. 

		\item \textbf{Updating} $\db^{k+1}$. 
The $\db$-subproblem in \eqref{subprob_d} can be written as	
		\begin{equation}
			\begin{aligned}
				\db^{k+1} = \underset{\substack{\db\in\Rbb^L}}{\argmin}\ \  &\frac{1}{2} (\wb^{k+1})^\top \sum_{\ell=1} d_\ell K_\ell \wb^{k+1} 
				+(\lambdab^{k})^\top D_{\yb}\sum_{\ell}d_\ell K_\ell\wb^{k+1}\\
				&+\frac{\rho_1}{2} \left\| \ub^{k+1}+D_{\yb}\sum_{\ell}d_\ell K_\ell\wb^{k+1} +b^{k+1}\yb-\oneb \right\|^2\\
				&+ (\thetab^{k})^\top \db +\frac{\rho_2}{2} \|\db-\zb^{k+1}\|^2 
				+\alpha^{k}\oneb^\top\db +\frac{\rho_3}{2}\left(\oneb^\top \db-1\right)^2.
			\end{aligned}
		\end{equation}
The stationary-point equation of the objective function with respect to each $d_\ell$ is 
			\begin{equation}\label{d_limit}
				\begin{aligned}
				& \frac{1}{2}(\wb^{k+1})^\top K_\ell\wb^{k+1}
				+\rho_1(D_{\yb}K_\ell\wb^{k+1})^\top(\lambdab^k/\rho_1 +\ub^{k+1}+b^{k+1} \yb-\oneb) \\ 
				& + \rho_1(K_\ell\wb^{k+1})^\top \sum_{t} d_t K_t \wb^{k+1}
				+\theta_\ell^k+\rho_2(d_\ell-z_\ell^{k+1})+\alpha^k+\rho_3 \left( \oneb^\top \db-1 \right) =0.
				\end{aligned}
			\end{equation}
%
			We can rewrite it in matrix form as 
			\begin{equation}\label{d_solution}
				(\rho_1 \mathfrak{A}^\top \mathfrak{A}+\rho_2 I+\rho_3\textbf{1}\cdot\textbf{1}^\top )\db= \vb-\thetab^k+\rho_2\zb^{k+1}+(\rho_3-\alpha^k)\oneb,
			\end{equation}
			where, 
			\begin{itemize}
			\item 	$\mathfrak{A}=[K_1\wb^{k+1}, \dots, K_L \wb^{k+1}]\in\Rbb^{m\times L}$,
			\item and $\vb=[v_1, \dots, v_L]^\top \in\Rbb^L$ with 
			$$v_\ell=-\frac{1}{2}(\wb^{k+1})^\top K_\ell\wb^{k+1}
			-\rho_1(D_\yb K_\ell\wb^{k+1})^\top( \lambdab^k/\rho_1 +\ub^{k+1}+b^{k+1}\yb-\oneb).$$
			\end{itemize}
		Clearly the coefficient matrix in \eqref{d_solution} is positive definite, and thus there exists a unique solution to the system of linear equations.
		However, it is possible that the solution vector could contain a negative component, which is not ideal because the components of $\db$ define the combination $\Kcal(\db)$ of the kernel matrices which must turn out to be positive definite. 
			In order to handle such a pathology, we propose an ad-hod recipe which projects the solution of the linear system to
			the standard $(L-1)$-simplex using the algorithm in \cite{wang2013projection}. Then motivated by the second formula in \eqref{update_z} and the equality constraint \eqref{constaint_indica_z},
			for $\ell\in \overline{S}_k$ we set
			\begin{equation}\label{update_d_Sk_bar}
				\db^{k+1}_{\overline{S}_k}=\zerob.
			\end{equation}
			
			\begin{remark}
				We have observed from our numerical simulations that the ad-hoc step of projection onto the simplex, which appears artificial, can significantly speed up convergence of the ADMM algorithm.
				
			In addition, \cite{shi2023admm} (the conference version of this paper) proposed a numerical procedure for the infinite-dimensional formulation \eqref{inf_opt_MKL_01.1}, and the major difference is about the subproblems of $\wb$ and $\db$. 
			Indeed, each function $f_\ell$ in \cite{shi2023admm} is parametrized by its values at the data points $\{\xb_i\}$, and the squared norm involves $K_\ell^{-1}$. In this way, the vector of function values $[f_\ell(\xb_1), \cdots, f_\ell(\xb_m)]^\top$ has to be updated for each $\ell\in\Nbb_L$. In contrast, the treatment in the current paper is more efficient and better conditioned since now only $\wb$ needs to be updated thanks to the representer theorem and no explicit inversion is needed for the kernel matrices $\{K_\ell\}$.
			Moreover, the update of $\db$ is also greatly simplified to linear algebra while previously one needs to solve a cubic polynomial equation for each component $d_\ell$.
			\end{remark}

			\item \textbf{Updating} $\thetab^{k+1}$. With the help of the working set \eqref{work_set_S_k}, the update of $\thetab$ in \eqref{dual_update_theta} can be simplified as:	
			\begin{equation}\label{update_theta}
					\thetab^{k+1}_{S_k}  =  \thetab^k_{S_k} + \rho_2(\db^{k+1}-\zb^{k+1})_{S_k},\quad \thetab^{k+1}_{\overline{S}_k}  =  \thetab^k_{\overline{S}_k}.
			\end{equation}
			That is to say, the components of $\thetab$ outside the current working set $S_k$ is not updated.

			\item \textbf{Updating} $\alpha^{k+1}$. See \eqref{dual_update_alpha}. We emphasize that the update of $\alpha$ must be carried out \emph{before} projecting $\db$ onto the simplex, since otherwise $\alpha$ would remain unchanged.

			\item \textbf{Updating} $\lambdab^{k+1}$.  Inspired by the property of the working set $T_*$ (see the line above \eqref{work_set_T*_equiv}), the update of $\lambdab$ in \eqref{dual_update_lambda} is simplified as follows:
			\begin{equation}\label{update_lambda}
				\begin{aligned}
					\lambdab^{k+1}_{T_k} = \lambdab^k_{T_k}+\rho_1\rb^{k+1}_{T_k},\quad \lambdab^{k+1}_{\overline{T}_k}=\zerob
				\end{aligned}
			\end{equation}
			where the vector of residuals $\rb^{k+1}=\ub^{k+1}+\Acal(\db^{k+1})\wb^{k+1} +b^{k+1}\yb-\oneb$. In other words, we remove the components of $\lambdab$ which are not in the current working set.
		\end{enumerate}

The update steps above are summarized in Algorithm~\ref{alg:admm_MKL_SVM}. Below we give a complexity analysis of the algorithm.
In each round of the ADMM updates,
	we have the following observations:
	\begin{itemize}
		\item In the update of $\ub_{k+1}$ by \eqref{update_u}, the main term is $\Acal(\db^k)\wb^{k}$ in \eqref{vec_s^k} which takes $\mathcal{O}(Lm^2)$ operations.
		\item To  update $\wb^{k+1}$, the dominant computation is forming $\Kcal(\db^k)=\sum_{\ell}d_\ell^k K_\ell$ and the solution of \eqref{stationa_eqn_w} whose time complexity is $\mathcal{O}(m^2 \max\{L, m\})$.
		\item    Similarly, the product $\Acal(\db^k)\wb^{k}$ is the most expensive computation in \eqref{update_b} to update $b^{k+1}$, and again its complexity is $\mathcal{O}(Lm^2)$.
		\item Updating $\zb^{k+1}$  by \eqref{update_z} has a complexity $\mathcal{O}(L)$.
		\item To update $\db^{k+1}$, one first computes the columns of $\mathfrak{A}$ with  $\Ocal(Lm^2)$ operations. These columns can then be used to compute $\vb$ with $\Ocal(Lm)$ operations.
		The product 
		$\mathfrak{A}^\top \mathfrak{A}$ needs $L^2m$ operations, and the solution of the linear system takes $\Ocal(L^3)$ operations. Therefore, the overall complexity is  $\mathcal{O}(L \max\{m^2, Lm, L^2\}$. 
		\item Computing $\thetab^{k+1}$ by \eqref{update_theta} takes $\mathcal{O}(L)$ operations which is the same as $\zb^{k+1}$, and updating $\alpha^{k+1}$ by \eqref{dual_update_alpha} takes $\Ocal(1)$ operations.
		\item The complexity of updating $\lambdab^{k+1}$ is $\mathcal{O}(Lm^2)$, the same as $b^{k+1}$, due to the computation of the residual vector $\rb^{k+1}$.
	\end{itemize}
	In summary, the time complexity for one round of updates of variables
	in Algorithm~\ref{alg:admm_MKL_SVM} 
	is
	$$\mathcal{O}(\max\{Lm^2, m^3, L^2m, L^3\}).$$
	If the number of candidate kernels $L$ is much smaller than $m$, then the above operations count reduces to $\Ocal(m^3)$ which is reasonable.
	The real computational times on different data sets will be reported in the next section. 

\begin{algorithm}
	\caption{ADMM for the MKL-$L_{0/1}$-SVM}
	\label{alg:admm_MKL_SVM}
	\begin{algorithmic}[1]
        \Require $C$, $\rho_1$, $\rho_2$, $\rho_3$, $\{\kappa_\ell\}$, \texttt{max\_iter}.
		\State Set the counter $k=0$, and initialize $(\wb^0, \db^0, b^0, \ub^0, \zb^0; \thetab^0, \alpha^0, \lambdab^0)$.
		\While{the terminating condition is not met and $k\le \texttt{max\_iter}$}
		\State Update $T_k$ in \eqref{work_set_T_k} and $\ub^{k+1}$ by \eqref{update_u}. 
		\State Update $\wb^{k+1}$ by the unique solution to \eqref{stationa_eqn_w}.
		\State Update $b^{k+1}$ by \eqref{update_b}.
		\State Update $S_k$ in \eqref{work_set_S_k} and $\zb^{k+1}$ by \eqref{update_z}.
		\State Update $\db^{k+1}$ by the unique solution to \eqref{d_solution} plus projection to the simplex and \eqref{update_d_Sk_bar}.
		\State Update $\thetab^{k+1}$ by \eqref{update_theta}.
		\State Update $\alpha^{k+1}$ by \eqref{dual_update_alpha}.
		\State Update $\lambdab^{k+1}$ by \eqref{update_lambda}.
		\State Set $k=k+1$.
		\EndWhile \\
		\Return the final iterate $(\wb^k, \db^k, b^k, \ub^k, \zb^k; \thetab^k, \alpha^k, \lambdab^k)$.
	\end{algorithmic}\label{ADMM-Solver}
\end{algorithm}

Unfortunately, we are not able to prove the convergence of Algorithm~\ref{alg:admm_MKL_SVM} as it seems very hard in general due to the nonsmoothness and nonconvexity of the optimization problem \eqref{opt_MKL_01.1}. However, we can give a characterization of the limit point \emph{if} the algorithm converges, see the next result. We remark that in practice Algorithm~\ref{alg:admm_MKL_SVM} converges well as we have done extensive simulations to be presented in the next section.

\begin{theorem}\label{thm_converg_implies_loc_opt}
	Suppose that the sequence 
	\begin{equation*}
		\{\Psi^k\} = \{(\wb^k, \db^k, b^k, \ub^k, \zb^k; \thetab^k, \alpha^k, \lambdab^k)\}
	\end{equation*}
	generated by the ADMM algorithm above has a limit point $\Psi^*=(\wb^*, \db^*, b^*, \ub^*, \zb^*; \thetab^*, \alpha^*, \lambdab^*)$. Then $(\wb^*, \db^*, b^*, \ub^*)$ is a P-stationary point in the sense of Definition~\ref{def_P2_stationary} with $\gamma=1/\rho_1$ and the Lagrange multipliers in the vector $(-\thetab^*, \alpha^*, \lambdab^*)$, and also a local minimizer of the problem \eqref{opt_MKL_01.1}.
\end{theorem}

	\begin{proof} 		
	Let $\{\Psi^j\}$ be a subsequence that converges to $\Psi^*$, and let $\sbf^*$ be the limit of the corresponding subsequence of $\{\sbf^k\}$ in \eqref{vec_s^k}.
	The simplest case is the dual update for $\alpha^k$. Since the constants $\rho_j$, $j=1, 2, 3$ are positive, we have $\oneb^\top\db^*=1$, which is the condition \eqref{inf_primal_constraint_02}, by taking the limit of the $\alpha$ update \eqref{dual_update_alpha}.

	Now look at the limit of the dual update for $\lambdab$. Define $T_*$ to be the limit working set in the sense of \eqref{work_set_T_k} with $\textbf{s}^k$ replaced by $\textbf{s}^*$. After taking the limit of the update formula \eqref{update_lambda}, we obtain
	\begin{equation}\label{limit_lambda}
		\rb^*_{T_*}=\zerob\quad \text{and} \quad\lambdab^*_{\overline{T}_*}=\zerob.
	\end{equation} 
	In addition, we check the limit of \eqref{update_u} and get $\ub^*_{T_*}=\zerob$ and $\ub^*_{\overline{T}_*}=\sbf^*_{\overline{T}_*}=(\ub^*-\rb^*-\lambdab^*/\rho_1)_{\overline{T}_*}$ where the last equality comes from the relation
	\begin{equation}\label{relation_s}
		\sbf^k=\ub^k-\rb^k-\lambdab^k/\rho_1.
	\end{equation} 
	With reference to \eqref{limit_lambda}, plain computation yields $\rb^*_{\overline{T}_*}=\zerob$ and furthermore $\rb^*=\zerob$. This is the condition \eqref{primal_constraint_03}. Now \eqref{relation_s} reduces to $\sbf^*=\ub^*-\lambdab^*/\rho_1$ in the limit. It is then not difficult to check that the update formula \eqref{update_u_prox} also holds in the limit which can be written as 
	\begin{equation}
		\ub^* = \prox_{\frac{C}{\rho_1}\|(\cdot)_+\|_0}(\ub^*-\lambdab^*/\rho_1).
	\end{equation}
	This is the condition \eqref{inf_stationary_cond_Lagrangian_u} with $\gamma=1/\rho_1$, and also explains why here we have used the same symbol $T_*$ as in \eqref{work_set_T*}.
	
	Next we take the limit of \eqref{w_limit}. Since we have shown \eqref{primal_constraint_03}, the last term in \eqref{w_limit} vanishes in the limit and we are left with \eqref{stationary_cond_Lagrangian_w}. Similarly, taking the limit of $\eqref{stationa_eqn_b}$ leads to \eqref{inf_stationary_cond_Lagrangian_b}.

	
	In order to describe the limit of \eqref{update_z}, let us introduce the limit working set $S_*$ in the sense of \eqref{work_set_S_k} with $d_\ell^k$ and $\theta^k_\ell$ replaced by $d_\ell^*$ and $\theta^*_\ell$, respectively. Then we have
	\begin{equation}\label{limit_z}
		\zb^*_{S_*} = (\db^*+\thetab^*/\rho_2)_{S_*}, \quad \zb^*_{\overline{S}_*} = \zerob.
	\end{equation}
	Moreover, taking the limit of $\thetab^k$ in \eqref{update_theta} gives $\db^*_{S_*} = \zb^*_{S_*}$ which is a part of the constraint \eqref{constaint_indica_z}, and together with \eqref{limit_z} implies $\thetab^*_{S_*}=\zerob$.
	In addition, the last update formula \eqref{update_d_Sk_bar} for $\db$ implies $\db^*_{\overline{S}_*}=\zerob$. These two relations establish the complementary slackness condition \eqref{inf_complem_slackness} and the desired equality $\db^*=\zb^*$. Also, the update strategy for $\db$ ensures that each $d_\ell^k$ is nonnegative which yields the limit \eqref{inf_primal_constraint_01}. In order to obtain \eqref{stationary_cond_Lagrangian_d}, we take the limit of the stationary-point equation \eqref{d_limit}, notice that the terms involving $\rho_1$, $\rho_2$, and $\rho_3$ vanish due to the established optimality conditions, and take \eqref{stationary_cond_Lagrangian_w} into account. Here we have a sign difference for the Lagrange multiplier $\thetab$.
	
	The only missing condition up till now is \eqref{inf_dual_constraint} in the index set $\overline{S}_*$. By definition, for $\ell\in\overline{S}_*$ it holds that $d_\ell^*+\theta_\ell^*/\rho_2\leq 0$ which implies $\theta_\ell^*\leq -\rho_2d_\ell^*\leq 0$. This is precisely \eqref{inf_dual_constraint} with a sign difference.
	Therefore, $(\wb^*, \db^*, b^*, \ub^*)$ is a P-stationary point with $\gamma=1/\rho_1$ and its local optimality is guaranteed by Theorem~\ref{thm_optimality}.
\end{proof}

\begin{remark}\label{rem:sparsity_kernel_combi}
	Similar to the working set $T_*$ in \eqref{work_set_T*} and the associated support vectors, the working set $S_k$ in \eqref{work_set_S_k} and its limit set $S_*$ renders \emph{sparsity} in the combination of the kernels $\{\kappa_\ell\}$ for the MKL task, because the constraint \eqref{d_l_simplex_02} can be interpreted as $\|\db\|_1=1$, an equality involving the $\ell_1$-norm, due to the nonnegativity condition \eqref{d_l_simplex_01}. Such an effect of sparsification can also be observed from our numerical example in the next section.
\end{remark}

\section{Simulations}\label{sec:sims}


       In this section,  we report results of numerical experiments using Matlab on a Dell laptop workstation with an Intel Core i7 CPU of 2.5 GHz on synthetic and real data sets to demonstrate the  effectiveness of the proposed MKL-\LzeroneSVM\ in comparison with
       the SimpleMKL approach of \cite{rakotomamonjy2008simplemkl}.
Our code can be found at  \url{https://github.com/shiyj27/MKL-L01-SVM/tree/main/MKL_L0-1_SVM-master/MKL_L0-1_SVM-master}.
    
        \textbf{(a) Stopping criterion.} In the implementation, we terminate Algorithm~\ref{alg:admm_MKL_SVM} if the iterate $(\wb^k, \db^k, b^k, \ub^k, \zb^k; \thetab^k, \alpha^k, \lambdab^k)$ satisfies the condition
        \begin{equation}\label{stopping_criteria}
        	\max \left\{\beta_1^k,\beta_2^k,\beta_3^k,\beta_4^k,\beta_5^k,\beta_6^k,\beta_7^k,\beta_8^k \right\}< \tol,
        \end{equation}
        where the number $\tol>0$ is the tolerance level and 
        \begin{equation}
        	\begin{array}{llll}
        		\beta_1^k:=\|\ub^k-\ub^{k-1}\|,\quad
        		& \beta_2^k:=\|\wb^k-\wb^{k-1}\|,\  
        		& \beta_3^k:=|b^k-b^{k-1}|,\quad
        		& \beta_4^k:=\|\zb^k-\zb^{k-1}\| \\
        		\beta_5^k:=\|\db^k-\db^{k-1}\|,\quad
        		& \beta_6^k:=\|\thetab^k-\thetab^{k-1}\|,\ 
        		& \beta_7^k:=|\alpha^k-\alpha^{k-1}|,\quad
        		& \beta_8^k:=\|\lambdab^k-\lambdab^{k-1}\|.
        	\end{array}    	
        \end{equation}
        Such a condition says, in plain words, that two successive iterates are sufficiently close. 
        
        \textbf{(b) Parameters setting.} In Algorithm~\ref{alg:admm_MKL_SVM}, the parameters $C$ and $\rho_1$ characterize the P-stationary-point condition \eqref{inf_stationary_cond_Lagrangian_u} (see the proof of Theorem~\ref{thm_converg_implies_loc_opt}), and the working set \eqref{work_set_T_k} which is related to the number of support vectors.
         In order to choose the parameters $C$, $\rho_1$, $\rho_2$ and $\rho_3$, the standard $10$-fold Cross Validation (CV) is employed on the training data, where the candidate values of $C$ and $\rho$'s are selected from the same set $\{2^{-2},2^{-1},\cdots,2^{8}\}$.  The parameter combination with the highest CV accuracy is picked out. 
         
         In addition, we choose $10$ Gaussian kernels $\{\kappa_\ell\}$
         with hyperparameters listed in Tabel~\ref{table:hyperparameters2} which are shared by both MKL-\LzeroneSVM\ and SimpleMKL for \emph{all} the data sets. 
         We also set the maximum number of iterations $\texttt{max\_iter}=10^3$ in Algorithm~\ref{alg:admm_MKL_SVM} and the tolerance level $\tol=10^{-3}$ in \eqref{stopping_criteria}.

         \begin{table}[H]
         	\begin{center}
         		\begin{tabular}{cccccccccc}
         			\toprule
         			$\sigma_1$&$\sigma_2$&$\sigma_3$&$\sigma_4$&$\sigma_5$&$\sigma_6$&$\sigma_7$&$\sigma_8$&$\sigma_9$&$\sigma_{10}$ \\
         			
         			\midrule
         			$0.1$&$0.2$&$0.3$&$0.5$&$0.7$&
         			$1$&$1.2$&$1.5$&$1.7$&$2$ \\
         			\bottomrule
         		\end{tabular}
         	         \caption{A quite arbitrary choice of the hyperparameters for $10$ Gaussian kernels.
         	}
         	\label{table:hyperparameters2}
         	\end{center}
         \end{table} 
         
         For the starting point, we set $\wb^0 = \ub^0=\lambdab^0=\zerob$, $\zb^0 = \thetab^0=\zerob$,
         $\alpha^0=0$ and $\db^0=\frac{1}{L} \oneb$, $b^0=1$ or $-1$. 
         The reason for such a choice is explained in the following. Let us recall the objective function $J(\wb,\db,b)$ in \eqref{opt_MKL_02.22}. Then we immediately notice that $J(\zerob, \frac{1}{L} \oneb, 1) = Cm_-$ and $J(\zerob, \frac{1}{L} \oneb, -1) = Cm_+$ where $m_+$ and $m_-$ denote the numbers of positive and negative components in the label vector $\yb$.
         Therefore, we should choose $(\wb^0, \db^0, b^0)$ such that $J(\wb^0, \db^0, b^0)\leq C\min\{m_+, m_-\}$.

         \textbf{(c) Evaluation criterion.} To evaluate the classification performance of our MKL-\LzeroneSVM, we use the testing accuracy (\texttt{ACC}): 
         \begin{equation*}
         	\texttt{ACC}:=1-\frac{1}{2m_{\test}}\sum_{j=1}^{m_{\test}} \left|\sign\left(\sum_{\ell}f^*_\ell(\xb^{\test})+b^*\right)-y^{\test}_j \right|,
         \end{equation*}
         where $\{(\xb^{\test}_j,y^{\test}_j):j=1,\cdots,m_{\test}\}$ contains the testing data.
         Here the quantity $\sum_{\ell}f^*_\ell(\xb^{\test})$ can be computed using \eqref{inf_stationary_cond_Lagrangian_f}. More specifically, for each $\ell\in\Nbb_L$ we can evaluate $f_\ell^\ast(\cdot)  =-d_\ell^\ast\sum_{i}\lambda_i^* y_i \kappa_\ell(\,\cdot\,, \xb_i)$ on the testing data using the convergent iterate produced by Algorithm~\ref{alg:admm_MKL_SVM}.
         
         
         \textbf{(d) Simulation results.}  
         We have conducted experiments on six real data sets\,\footnote{These data sets are downloaded from the UCI repository  \url{https://archive.ics.uci.edu/}.
         }.
         For each data set, standard feature normalization is performed. Then the Algorithm~\ref{alg:admm_MKL_SVM} is run $30$ times with different training and testing sets ($70\%$ of the data are randomly selected for training and the rest $30\%$ for testing), and the last $20$ \texttt{ACC}s are averaged. 
         In order to avoid possibly bad local minimizer, we use a ``warm start'' strategy which initializes the algorithm with the convergent iterate of the previous run, hoping that the performance of classifier eventually becomes stable.
         For comparisons, we run the SimpleMKL algorithm\,\footnote{The code for SimpleMKL is found at  \url{https://github.com/shiyj27/SimpleMKL/tree/main} which is adapted for simulations reported in this paper from the source at \url{https://github.com/maxis1718/SimpleMKL}.}
         \begin{itemize}
         	\item[1)] with the same ten Gaussian kernels (see Table~\ref{table:hyperparameters2} for the hyperparameters), 
         	\item[2)] and with additional Gaussian kernels for each component of the data vector\,\footnote{For a data set with $\xb=[x_1, \dots, x_n]^\top\in\Rbb^n$, one includes the Gaussian kernels on $\Rbb^n \times \Rbb^n$ with hyperparameters in Table~\ref{table:hyperparameters2}. In addition, for each $j=1, \dots, n$, Gaussian kernels on $\Rbb\times\Rbb$ in the variable $(x_j, y_j)$ with  the same ten hyperparameters are used. Therefore, the total number of candidate kernels is $L=10(n+1)$.}.
         \end{itemize}
         The results are collected in Table~\ref{tab:Accuracy} where one can see that 
         our method achieves the best \texttt{ACC} on two data sets.
         The overall performances of the three methods in terms of the \texttt{ACC} are rather comparable, and the
         difference between the \texttt{ACC} of MKL-$L_{0/1}$-SVM and the best one among the three methods is within $3\%$ on each data set.
         In terms of the sparsity in the kernel combination, MKL-$L_{0/1}$-SVM seems the best. On the other hand, SimpleMKL with $L=10$ outperforms the rest in the computational time which falls within expectation since the ADMM algorithm in general converges with a sublinear rate. At last, SimpleMKL with additional kernels for individual features wins in terms of the number of support vectors (\texttt{NSV}, sparsity in the data usage) which is equal to $|T_*|$.


         Figure~\ref{fig:result_iono_Wpbc} depicts changes of $\db$ with respect to the number of ADMM iterations
         for the three largest components. For the Ionosphere data set, in particular, the largest component of $\db$ reaches $1$ when the algorithm converges and the rest components are all zero.


\begin{table}[ht]
	\centering
	Ionosphere \quad$m=351$\quad $n=33$
	\begin{tabular}{|>{\centering\arraybackslash}p{4cm}||>{\centering\arraybackslash}p{2cm}|>{\centering\arraybackslash}p{2cm}|>{\centering\arraybackslash}p{2cm}|>{\centering\arraybackslash}p{2cm}|}
		\hline
		Algorithm & \# Kernels& \texttt{ACC} (\%) & Time (s) & \texttt{NSV}  \\
		\hline
		SimpleMKL ($L=340$) & 28.6$\pm$5.2 & 91.1$\pm$1.9 &12.7$\pm$3.4 & \textbf{115$\pm$4.8} \\
		\hline
		MKL-\LzeroneSVM & \textbf{1$\pm$0} & 89.5$\pm$2.2 &4.6$\pm$2.7 & 193$\pm$4.3  \\
		\hline
		SimpleMKL ($L=10$) & \textbf{1$\pm$0} & \textbf{92.4$\pm$2.2} &\textbf{0.08$\pm$0.02} & 162$\pm$3.4  \\
		\hline
	\end{tabular}
	
	~\\
	
	
	\centering
	Sonar \quad$m=208$\quad $n=60$
	\begin{tabular}{|>{\centering\arraybackslash}p{4cm}||>{\centering\arraybackslash}p{2cm}|>{\centering\arraybackslash}p{2cm}|>{\centering\arraybackslash}p{2cm}|>{\centering\arraybackslash}p{2cm}|}
		\hline
		Algorithm & \# Kernels& \texttt{ACC} (\%) & Time (s) & \texttt{NSV}  \\
		\hline
		SimpleMKL ($L=610$) & 43.8$\pm$6.6 &  \textbf{78.8$\pm$5.3} &12.6$\pm$2.3 & \textbf{112$\pm$6.3} \\
		\hline
		MKL-\LzeroneSVM & \textbf{1$\pm$0} & 76.5$\pm$3.7 &1.04$\pm$0.08 & 145$\pm$1.5 \\
		\hline
		SimpleMKL ($L=10$) & \textbf{1$\pm$0} & 63.3$\pm$3.1 &\textbf{0.05$\pm$0.02} & 145$\pm$0.5 \\
		\hline
	\end{tabular}

	~\\
	
	\centering
	Wpbc  \quad$m=198$\quad $n=33$
	\begin{tabular}{|>{\centering\arraybackslash}p{4cm}||>{\centering\arraybackslash}p{2cm}|>{\centering\arraybackslash}p{2cm}|>{\centering\arraybackslash}p{2cm}|>{\centering\arraybackslash}p{2cm}|}
		\hline
		Algorithm & \# Kernels& \texttt{ACC} (\%) & Time (s) & \texttt{NSV}  \\
		\hline
		SimpleMKL ($L=340$) & 28.1$\pm$6 &  76.6$\pm$0.4 &3.01$\pm$0.48 & \textbf{124$\pm$8} \\
		\hline
		MKL-\LzeroneSVM & \textbf{3.4$\pm$2.2} & \textbf{76.8$\pm$0.6} &0.65$\pm$0.15 & 138$\pm$0\\
		\hline
		SimpleMKL ($L=10$) & 8.9$\pm$0.73 & 76.7$\pm$0 &\textbf{0.04$\pm$0.02} & 137$\pm$2.4 \\
		\hline
	\end{tabular}
	
	
	~\\
	
	\centering
	Pima  \quad$m=768$\quad $n=8$
	\begin{tabular}{|>{\centering\arraybackslash}p{4cm}||>{\centering\arraybackslash}p{2cm}|>{\centering\arraybackslash}p{2cm}|>{\centering\arraybackslash}p{2cm}|>{\centering\arraybackslash}p{2cm}|}
		\hline
		Algorithm & \# Kernels& \texttt{ACC} (\%) & Time (s) & \texttt{NSV}  \\
		\hline
		SimpleMKL ($L=90$) & 30.4$\pm$5.3 &  75.0$\pm$1.9 &6.4$\pm$0.69 & \textbf{348$\pm$15} \\
		\hline
		MKL-\LzeroneSVM & 2.4$\pm$0.8 & 73.3$\pm$4.1 &106$\pm$58 & 490$\pm$34.3\\
		\hline
		SimpleMKL ($L=10$) & \textbf{1$\pm$0} & \textbf{75.9$\pm$2.1} &\textbf{0.24$\pm$0.03} & 361$\pm$4.2 \\
		\hline
	\end{tabular}
	
	~\\
	
	\centering
	Liver  \quad$m=579$\quad $n=9$
	\begin{tabular}{|>{\centering\arraybackslash}p{4cm}||>{\centering\arraybackslash}p{2cm}|>{\centering\arraybackslash}p{2cm}|>{\centering\arraybackslash}p{2cm}|>{\centering\arraybackslash}p{2cm}|}
		\hline
		Algorithm & \# Kernels& \texttt{ACC} (\%) & Time (s) & \texttt{NSV}  \\
		\hline
		SimpleMKL ($L=100$) & 25.2$\pm$7.4 &  71.3$\pm$0 &4.7$\pm$0.58 & \textbf{400$\pm$6} \\
		\hline
		MKL-\LzeroneSVM & \textbf{1$\pm$0} & \textbf{71.7$\pm$0.8} &45$\pm$5.5 & 405$\pm$0\\
		\hline
		SimpleMKL ($L=10$) & 10$\pm$0.7 & 71.3$\pm$0 &\textbf{0.43$\pm$0.05} & 405$\pm$0\\
		\hline
	\end{tabular}
	
	~\\
	
	\centering
	Haberman  \quad$m=305$\quad $n=3$
	\begin{tabular}{|>{\centering\arraybackslash}p{4cm}||>{\centering\arraybackslash}p{2cm}|>{\centering\arraybackslash}p{2cm}|>{\centering\arraybackslash}p{2cm}|>{\centering\arraybackslash}p{2cm}|}
		\hline
		Algorithm & \# Kernels& \texttt{ACC} (\%) & Time (s) & \texttt{NSV}  \\
		\hline
		SimpleMKL ($L=40$) & 16.7$\pm$3.4 &  73.7$\pm$1.0 &0.19$\pm$0.04 & 209$\pm$1.8\\
		\hline
		MKL-\LzeroneSVM & \textbf{1$\pm$0} & 72.8$\pm$1.5 &44.5$\pm$2.2 & \textbf{197$\pm$2}\\
		\hline
		SimpleMKL ($L=10$) & 8.4$\pm$2.0 & \textbf{73.9$\pm$0} &\textbf{0.09$\pm$0.02} & 210$\pm$1.5 \\
		\hline
	\end{tabular}
	\caption{ Average performance measures for MKL-\LzeroneSVM, SimpleMKL with ten Gaussian kernels (last row in each subtable), and SimpleMKL with ten Gaussian kernels of the feature vector and of individual features (second row in each subtable).}
	\label{tab:Accuracy}
\end{table}

         \begin{figure}
        	\begin{subfigure}[H]{0.5\textwidth}
        		\centering
        		\includegraphics[width= 0.95\linewidth]{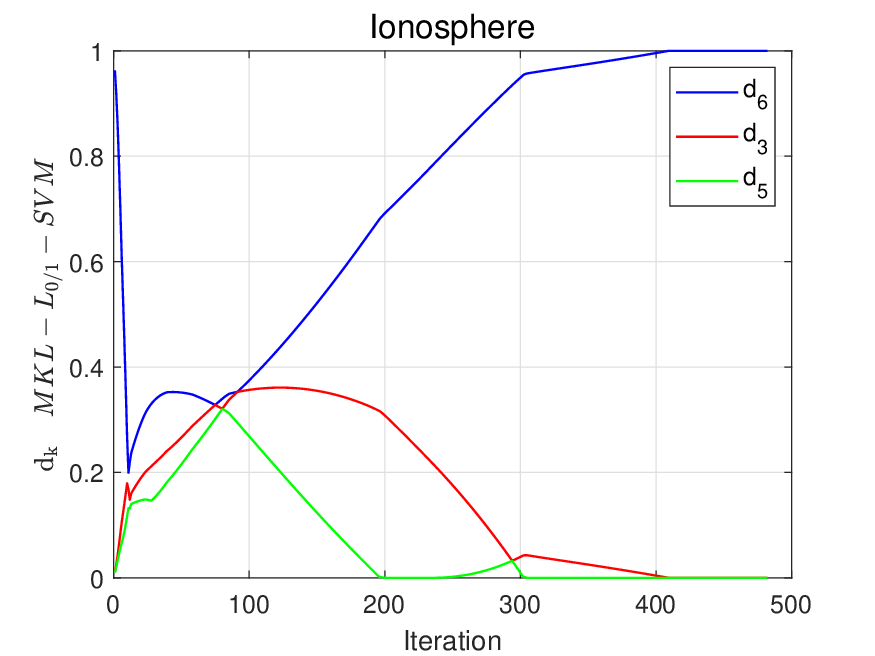}
        	\end{subfigure}
        	\hfill
        	\begin{subfigure}[H]{0.5\textwidth}
        		\centering
        		\includegraphics[width= 0.95\linewidth]{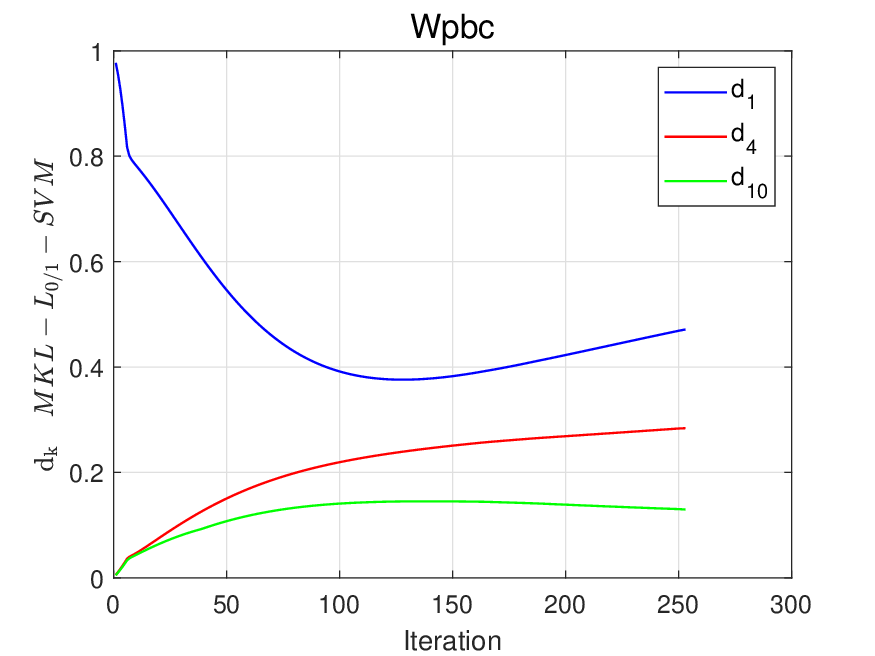}
        	\end{subfigure}
        	\caption{Evolution of the three largest weights $d_\ell$ in MKL-$L_{0/1}$-SVM on two data sets. \emph{Left}: Ionosphere, \emph{Right}: Wpbc.}
        	\label{fig:result_iono_Wpbc}
        \end{figure}

\section{Conclusion}\label{sec:Conclusion}   



In this paper, we have extended the $L_{0/1}$-SVM to a multiple kernel learning setting where the minimization of a regularized $(0, 1)$-loss function is accomplished together with the optimal combination of some given kernel functions. 
Despite the fact that the corresponding optimization problem is nonsmooth and nonconvex, we have shown that the global and local minimizers can be characterized by a set of KKT-like optimality conditions. 
For the numerical solution of the MKL-$L_{0/1}$-SVM problem, an efficient ADMM algorithm has been proposed to obtain a local optimum.
Simulation results on real data sets have manifested the effectiveness of our theory and algorithm.

Regarding future studies, one should deal with the convergence question of Algorithm~\ref{alg:admm_MKL_SVM} in view of the works \cite{li2015global,hong2016convergence,wang2019global,boct2020proximal} from the optimization community. 


     
\acks{The authors would like to thank Mr. Jiahao Liu for his assistance in the Matlab implementation of Algorithm~\ref{ADMM-Solver}.
	
	This work was support supported in part by Shenzhen Science and Technology Program (Grant No.~202206193000001-20220817184157001), the Fundamental Research Funds for the Central Universities, and the ``Hundred-Talent Program’’ of Sun Yat-sen University. 
    }

\vskip 0.2in
\bibliography{references}

\end{document}